\documentclass[10pt,twocolumn,letterpaper]{article}

\usepackage{cvpr}
\usepackage{times}
\usepackage{epsfig}
\usepackage{graphicx}
\usepackage{amsmath}
\usepackage{amssymb}

\usepackage{cases}
\usepackage[linesnumbered,ruled,vlined,lined,boxed,commentsnumbered]{algorithm2e}
\usepackage{subfigure}
\usepackage{indentfirst}
\usepackage{multirow}
\usepackage{boldline}
\usepackage{dsfont}

\usepackage{enumitem}
\setitemize{noitemsep,topsep=0pt,parsep=0pt,partopsep=0pt}
\usepackage{mathtools}
\usepackage{nccmath}

\usepackage[table]{xcolor}
\usepackage{makecell}

\usepackage{enumitem}
\setlist{nolistsep}

\usepackage{amsthm}
\theoremstyle{definition}
\newtheorem{definition}{Definition}

\newtheorem{proposition}{Proposition}

\usepackage{url}

\usepackage[pagebackref=true,breaklinks=true,letterpaper=true,colorlinks,bookmarks=false]{hyperref}

\def\bitter{bitter} 
\def\sweet{sweet} 
\def\taster{taster}
\def\tasters{tasters}

\def\padv{$\hat{\rho_{adv}}$}

\cvprfinalcopy 

\def\cvprPaperID{.} 

\ifcvprfinal\fi
\begin{document}

\title{Measuring Dataset Granularity}

\author{
  \hspace{-1.5mm}Yin Cui$^{1,2}$\thanks{Equal contribution. Yin Cui is now at Google Research.},\hspace{5pt}
  Zeqi Gu$^{1}$\footnotemark[1],\hspace{5pt}
  Dhruv Mahajan$^{3}$,\hspace{2pt}
  Laurens van der Maaten$^{3}$,\hspace{2pt}
  Serge Belongie$^{1,2}$,\hspace{2pt}
  Ser-Nam Lim$^{3}$\\
$^{1}$Cornell University\hspace{30pt}$^{2}$Cornell Tech\hspace{30pt}$^{3}$Facebook AI
}

\maketitle

\begin{abstract}
Despite the increasing visibility of fine-grained recognition in our field, ``fine-grained'' has thus far lacked a precise definition. 
In this work, building upon clustering theory, we pursue a framework for measuring dataset granularity. 
We argue that dataset granularity should depend not only on the data samples and their labels, but also on the distance function we choose.
We propose an axiomatic framework to capture desired properties for a dataset granularity measure and provide examples of measures that satisfy these properties.
We assess each measure via experiments on datasets with hierarchical labels of varying granularity.
When measuring granularity in commonly used datasets with our measure, we find that certain datasets that are widely considered fine-grained in fact contain subsets of considerable size that are substantially more coarse-grained than datasets generally regarded as coarse-grained.
We also investigate the interplay between dataset granularity with a variety of factors and find that fine-grained datasets are more difficult to learn from, more difficult to transfer to, more difficult to perform few-shot learning with, and more vulnerable to adversarial attacks.

\end{abstract}

\section{Introduction}

\begin{figure*}[t]
\begin{center}
\includegraphics[width=1.0\linewidth]{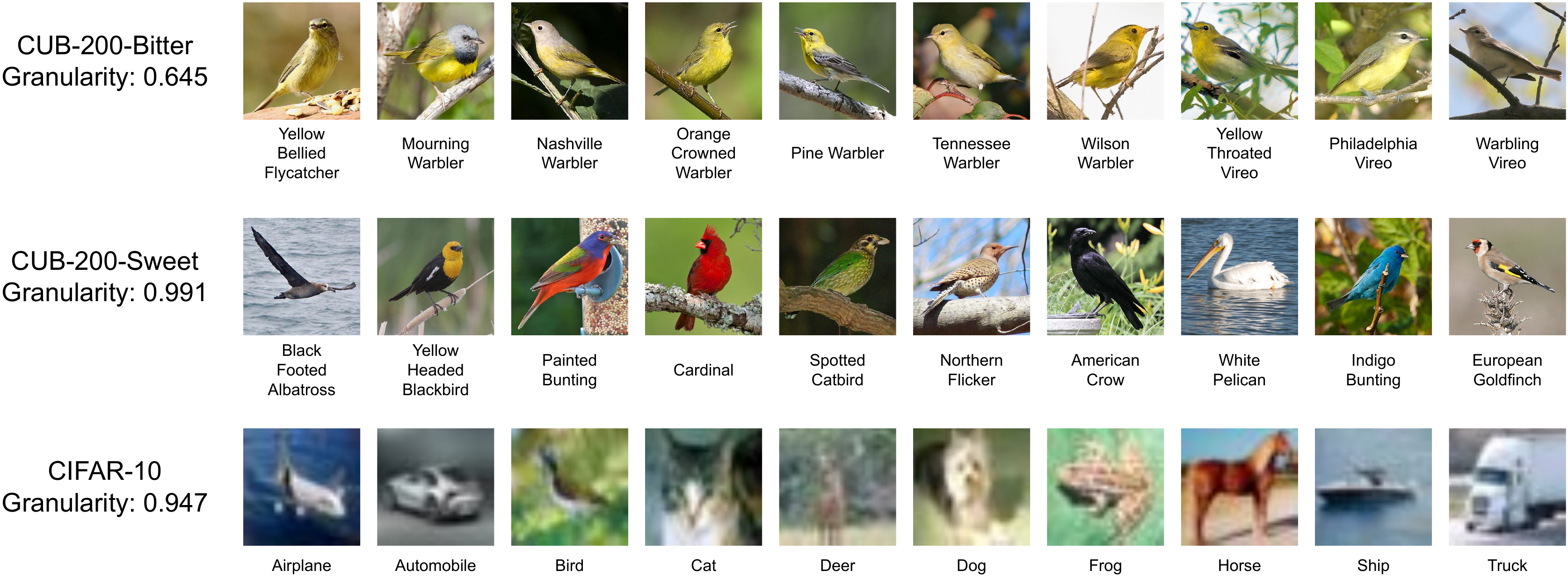}
\end{center}
\caption{Using our measure of dataset granularity, which ranges from $0$ for the finest granularity to $1$ for the coarsest, we show two subsets from CUB-200 with 10 classes that are fine-grained (CUB-200-Bitter) and coarse-grained (CUB-200-Sweet). CIFAR-10, which is widely considered coarse-grained, is more fine-grained than CUB-200-Sweet. We represent a class by a randomly selected image from that class.}
\label{fig:overview}
\vspace{-2mm}
\end{figure*}

Recent advances in deep learning~\cite{lecun2015deep} have fueled remarkable progress in visual recognition~\cite{alexnet, resnet}.
However, fine-grained recognition on real-world data still remains challenging~\cite{inaturalist}.
Fine-grained recognition is often regarded as a sub-field of object recognition that aims to distinguish subordinate categories within entry-level categories.
Examples include recognizing natural categories such as bird species~\cite{cub200} and dog breeds~\cite{stanford_dog}, or man-made categories such as fashion attributes~\cite{liu2016deepfashion} and car make \& model~\cite{stanford_car}.
There are also recent efforts focusing on fine-grained recognition in real-world data that are highly imbalanced~\cite{mahajan2018exploring,cui2019class}.
Despite the increasing visibility of fine-grained recognition as a sub-field, the designation has thus far lacked a precise definition.

A clear definition and quantitative evaluation of dataset granularity can not only give us a better understanding of current algorithms' performances on various tasks, but also provide a step toward a compact characterization of large-scale datasets that serve as training data for modern deep neural networks. 
Such a characterization can also provide a foundation for inter-dataset comparisons, independent of the domain semantics.
Given the nature of the problem, instead of regarding whether a dataset is fine-grained or not as a True-or-False binary question, we pursue a framework to measure dataset granularity in a continuous scale.

Dataset granularity depends on two factors: the ground truth labeling and the distance function. 
The notion of dataset granularity only applies when the dataset is labeled. For unlabeled data, in extreme cases, we could interpret samples as belonging to a single super class covering the entire dataset, or as each belonging to an instance-level class, a fact that makes granularity not well-defined.
As for the distance function, it includes a feature representation method and the metric to calculate distance between features. 
Examples of distance functions include Euclidean distance between features extracted from deep networks, Hamming distance between binary encoding of attributes, \etc.
The granularity of a dataset will be different if we change how we represent samples and calculate their distances.
This is aligned with the notion of fine-grained in human perception.
Visually similar bird species are difficult to distinguish for ordinary people but easy for bird experts~\cite{nabirds}.
People with congenital amusia have difficulty processing fine-grained pitch in music and speech~\cite{tillmann2011fine}.
In this respect, using different distance functions on a dataset is analogous to different levels of domain expertise.

We define a dataset granularity measure as a function that maps a labeled dataset and a distance function to a value indicating the dataset granularity, wherein low granularity means fine-grained and high granularity means coarse-grained.
Considering there might be multiple reasonable measures, we seek a set of broad, intuitive notions for which we have a consensus on what a dataset granularity measure should capture.
Inspired by the seminal work of Kleinberg~\cite{kleinberg2003impossibility} and Ben-David and Ackerman~\cite{ben2009measures} on clustering theory, we propose an axiomatic framework comprising three desired properties for a dataset granularity measure: granularity consistency, isomorphism invariance, and scale invariance.
We show these properties are self-consistent by providing a few examples of dataset granularity measures.

To choose the optimal measure among our proposals, we assess their quality on datasets with predefined hierarchical labels.
Intuitively, when we merge subordinate classes into a higher-level super-class, the dataset is less fine-grained thus the granularity should be strictly larger. We use this cue as an oracle for empirical evaluation of candidate measures. Meanwhile, we perform robustness studies by using various deep networks pre-trained on ImageNet as distance functions.
Among a wide-range of datasets, we find that the relative ranking of dataset granularity is insensitive to the type of granularity measure and the change of model architecture in distance function. Therefore, unless otherwise mentioned, we use the best measure candidate with a fixed network in distance function (\ie, ResNet-50 pre-trained on ImageNet) for all the experiments.

The first finding of our experiments is that, surprisingly, current datasets that are widely considered fine-grained in fact always contain a subset of a large number of classes with granularity higher than some other datasets generally considered coarse-grained.
Figure~\ref{fig:overview} shows that based on our proposed dataset granularity measure, a subset from CUB-200~\cite{cub200} has higher granularity than CIFAR-10~\cite{cifar}. For each fine-grained dataset, we are able to extract two subsets with drastically different granularity, named \textit{\bitter{} \taster{}} for the fine-grained subset and \textit{\sweet{} \taster{}} for the coarse-grained one. 
We then use two \tasters{} to study the interplay between dataset granularity and various factors. Specifically, we run existing benchmarks of (1) image classification, (2) transfer learning, (3) few-shot learning and (4) adversarial attack. 
Our experiments show that as dataset granularity decreases to become more fine-grained, the difficulty of classification, transfer learning, and few-shot learning increases, and the difficulty of adversarial attack decreases.
Our key contributions are as follows.

First, \textbf{a theoretical framework} that defines axioms a dataset granularity measure should satisfy. 
To the best of our knowledge, this is the first attempt to quantitatively define dataset granularity. 
We provide examples of granularity measures and empirically select the best candidate.

Second, we show how to extract \textbf{fine-grained \bitter{} and coarse-grained \sweet{} \tasters{}} from a given dataset and illustrate the significant granularity difference between them. 

Third, we perform \textbf{a comprehensive empirical study} on the characteristics of dataset granularity and their impacts on various tasks including image recognition, transfer learning, few-shot learning, and adversarial attack. 
Observing the large performance gap between the two \tasters{}, we believe the \bitter{} \taster{} can provide a more strict fine-grained benchmark without extra labeling effort, which is especially useful for existing datasets with saturated performance.


\section{Related Work}

\textbf{Learning from Fine-Grained Visual Data.}
Recent work on learning from fine-grained visual data focus on designing a deep network in an end-to-end fashion that incorporates modules for better capturing subtle visual differences.
Examples of such modules include non-linear feature interactions~\cite{bilinearcnn,cbp,kernel_pooling} and attention mechanisms~\cite{jaderberg2015spatial,two-level_attention,fu2017look,multi-attention_fgvc,sun2018multi,luo2019cross,zhang2019learning}.
Moreover, efforts have been made to understand the characteristics of learning from fine-grained visual data.
Some observations include: (1) deep networks for fine-grained recognition benefit from higher input image resolution~\cite{jaderberg2015spatial,bilinearcnn,cui2018large}; 
(2) transfer learning to domain-specific fine-grained datasets benefits from networks pre-trained on an adaptively selected subset from the source domain~\cite{ge2017borrowing,cui2018large,ngiam2018domain};
(3) ImageNet pre-training appears to have only marginal benefits for fine-grained datasets when their labels are not well-represented in ImageNet~\cite{kornblith2018better};
(4) Low-shot learning is more difficult on fine-grained dataset, where the dataset granularity is empirically measured by the mean Euclidean distance between class centroids~\cite{kozerawski2018clear}.
Our work differs from them by focusing on how to measure dataset granularity and then performing extensive quantitative studies using the proposed measure.

\textbf{Clustering Theory and Evaluation.}
The notion of clustering arises naturally whenever one aims to group similar objects together.
Despite this intuitively compelling goal, a unified theoretical framework for clustering is difficult to find~\cite{kleinberg2003impossibility}.
However, it is relatively easy to construct a theoretical framework for clustering quality measures~\cite{ben2009measures}.
Commonly used metrics for clustering quality evaluation can be divided into two groups: supervised and unsupervised.
Supervised clustering measures including Normalized Mutual Information (NMI), the Rand index~\cite{hubert1985comparing}, the V-measure~\cite{rosenberg2007v} and the Fowlkes-Mallows index (FMI)~\cite{fowlkes1983method}
compare the clustering assignment against ground-truth labels.
Unsupervised clustering measures~\cite{silhouettes,calinski1974dendrite,davies1979cluster,baker1975measuring,hubert1976general} evaluate the intrinsic quality of clustering results.
In our work, build upon seminal work in \cite{kleinberg2003impossibility,ben2009measures}, we regard a labeled dataset as a clustering assignment and develop a theoretical framework that defines desired properties for a dataset granularity measure.
Note that some commonly used unsupervised clustering measures, including the Silhouette index~\cite{silhouettes}, the Calinski-Harabaz index~\cite{calinski1974dendrite} and the Davies-Bouldin index~\cite{davies1979cluster}, do not satisfy our desired properties.
Please refer to Desgraupes~\cite{clusteringindices} for a comprehensive review of clustering evaluation metrics.

\section{Measuring Dataset Granularity}

We introduce an axiomatic framework for measuring dataset granularity by starting with the problem formulation and mathematical notations.
Then we define three axiomatic properties that a dataset granularity measure should satisfy and show these properties are self-consistent by providing a few examples of dataset granularity measure.
Finally, we assess granularity measures and study their robustness with respect to the distance function.

\subsection{Problem Formulation and Notations}

We denote a labeled dataset as $\mathcal{S} = (\mathcal{X}, \mathcal{Y})$, where 
$\mathcal{X} = \{x_i\}_{i=1}^n$ is a set of samples and $\mathcal{Y} = \{y_i \in \{1, 2, \dots, k\}\}_{i=1}^n$ is a set of corresponding labels~\footnote{We assume there is exactly one ground-truth label for a sample.}. $n$ denotes the number of samples and $k$ is the number of classes.
The set of labels $\mathcal{Y}$ divides samples $\mathcal{X}$ into $k$ classes $C = \{C_1, C_2, \dots, C_k\}$, where $C_i = \{x_j\ |\ y_j = i,\ x_j \in \mathcal{X},\ y_j \in \mathcal{Y}\}$, $C_i \cap C_j = \emptyset$ for $i \neq j$ and $\cup_{i=1}^k C_i = \mathcal{X}$.

A distance function $d: \mathcal{X} \times \mathcal{X} \mapsto \mathbb{R}$ defines the similarity between two samples, such that $\forall\ x_i, x_j \in \mathcal{X}$, $d(x_i, x_j) \geq 0$, $d(x_i, x_j) = d(x_j, x_i)$ and $d(x_i, x_j) = 0$ if and only if $x_i = x_j$. Note that we do not require distance functions to be metrics satisfying the triangle inequality: $d(x_i, x_j) + d(x_j, x_k) \geq d(x_i, x_k), \forall\ x_i, x_j, x_k \in \mathcal{X}$.
In our experiments, we use distance functions as the Euclidean distance between features extracted from deep networks.

We define a \textit{dataset granularity measure} as a function $g: \mathcal{S} \times d \mapsto \mathbb{R}$ that maps a labeled dataset and a distance function to a real number indicating the dataset granularity.
Smaller number of $g(\mathcal{S}, d)$ means finer granularity.

\subsection{Properties of Dataset Granularity Measures}

We propose a set of three desired properties as axioms that a dataset granularity measure ought to satisfy.
The idea is to characterize what should be a dataset granularity measure and prevent trivial measures.

Intuitively, a labeled dataset is fine-grained if the distance between intra-class samples is relatively large compared with the distance between inter-class samples.
If we reduce the intra-class distances and enlarge the inter-class distances, the dataset should be less fine-grained and have a larger granularity.
This could be done by using a distance function that better represents the data and makes samples well clustered.
We want to have such a property that captures the concept of making a dataset less fine-grained.

Formally, we define $d'$ as a \textit{granularity consistent transformation} of $d$ if $\forall\ x_i, x_j \in \mathcal{X}$ in the same class of $C$, we have $d'(x_i, x_j) \leq d(x_i, x_j)$; and $\forall\ x_i, x_j \in \mathcal{X}$ in different classes of $C$, we have $d'(x_i, x_j) \geq d(x_i, x_j)$.

\begin{definition}[Granularity Consistency]
\label{def:consistency}
Suppose distance function $d'$ is a granularity consistent transformation of distance function $d$. For any $d'$, $g(\mathcal{S}, d') \geq g(\mathcal{S}, d)$.
\end{definition}

The definition of granularity consistency is inspired by the consistency property proposed in Kleinberg~\cite{kleinberg2003impossibility} and Ben-David and Ackerman~\cite{ben2009measures}.

Another nice property we would like to have is that a dataset granularity measure should remain the same when we permute the class indices.
Formally speaking, we define $\mathcal{S}' = (\mathcal{X}', \mathcal{Y}')$ as an \textit{isomorphic transformation} of $\mathcal{S} = (\mathcal{X}, \mathcal{Y})$ if $\mathcal{X}' = \mathcal{X}$ and $\mathcal{Y}' = \{\sigma(y_i)\}_{i=1}^n$, where $\mathcal{Y} = \{y_i \in \{1, 2, \dots, k\}\}_{i=1}^n$ and $\sigma(.)$ is a permutation of $\{1, 2, \dots, k\}$.
We require the measure to be invariant to any isomorphic transformations.

\begin{definition}[Isomorphism Invariance]
\label{def:isomorphism}
Suppose dataset $\mathcal{S}'$ is a isomorphic transformation of dataset $\mathcal{S}$. For any $\mathcal{S}'$, $g(\mathcal{S}', d) = g(\mathcal{S}, d)$.
\end{definition}

Figure~\ref{fig:transformation} illustrates the isomorphic transformation, under which the granularity should remain unchanged, and the granularity consistent transformation, under which the granularity should be larger.

Finally, a measure should be scale invariant with respect to the distance function.

\theoremstyle{definition}
\begin{definition}[Scale Invariance]
\label{def:scale}
For any distance function $d$ and any $\alpha > 0$, we have $g(\mathcal{S}, d) = g(\mathcal{S}, \alpha d)$.
\end{definition}

The scale invariance property requires the dataset granularity measure to be invariant to changes in the units of distance function.
This property can be easily satisfied by normalizing the feature vector before calculating distance (\eg, $\ell_2$-normalization for Euclidean distance).





\begin{figure}[t]
\begin{center}
\subfigure[Granularity Consistent Transformation]{
\includegraphics[width=0.8\columnwidth]{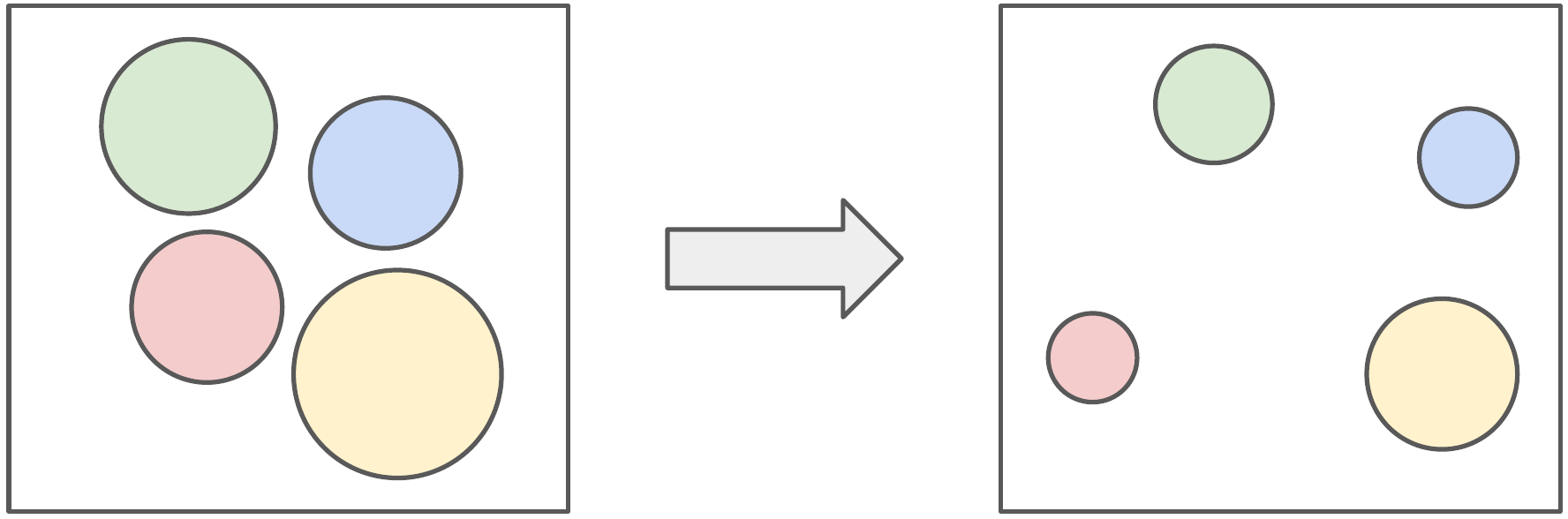}
   \label{fig:transformation_gamma}}
\subfigure[Isomorphic Transformation]{
\includegraphics[width=0.8\columnwidth]{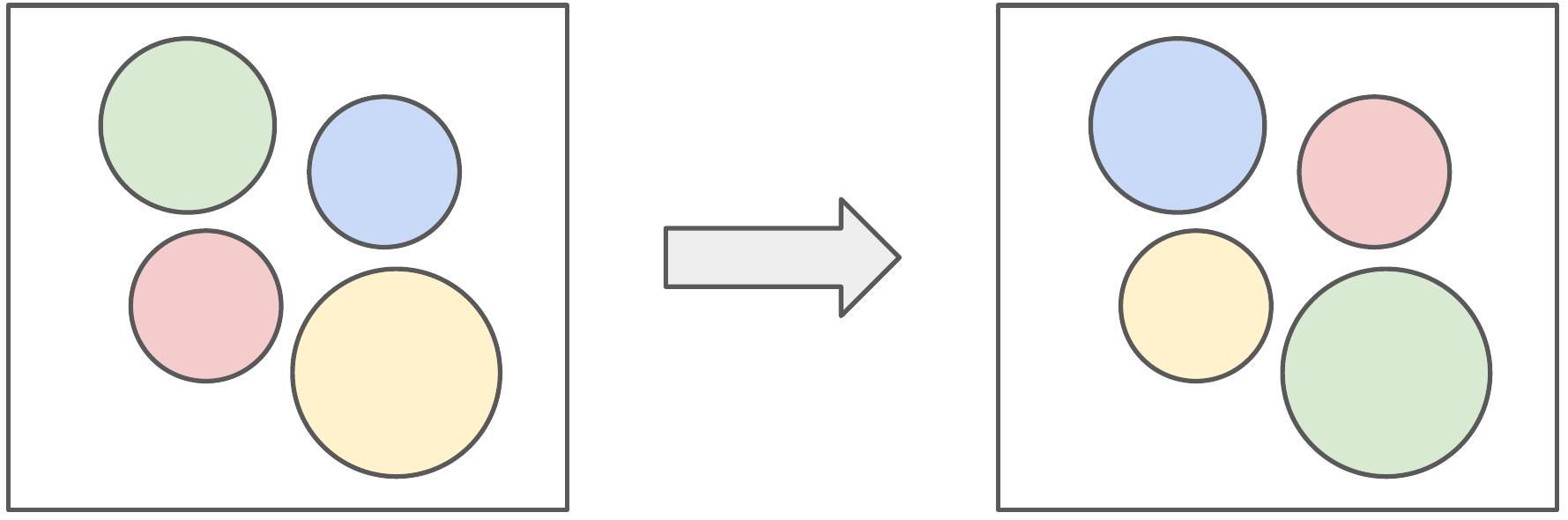}
   \label{fig:transformation_iso}}
\end{center}
\caption{Illustration of granularity consistent transformation and isomorphic transformation on a dataset with $4$ classes, denoted by $4$ disks with different sizes and colors. We show the original dataset $\mathcal{S}$ on the left and the transformed dataset $\mathcal{S}'$ on the right. In granularity consistent transformation, we reduce the within-class distances and enlarge the between-class distances. In isomorphic transformation, we permute the class indices.}
\label{fig:transformation}
\end{figure}

\subsection{Examples of Dataset Granularity Measures}

To show that granularity consistency, isomorphism invariance and scale invariance are self-consistent, we present examples of dataset granularity measures that satisfy all these desired properties. 

\textbf{The Fisher criterion (Fisher).}
Inspired by the \textit{Fisher criterion}~\cite{fisher1936use} used in Fisher's linear discriminant, we use the ratio of mean between-class distance and mean within-class distance as a granularity measure:
\begin{equation}
    \text{Fisher}(\mathcal{S}, d) = \frac{\frac{2}{k(k-1)}\sum_{i \neq j} d(c_i, c_j)}{\frac{1}{n} \sum_{i=1}^n d(x_i, c_{x_i})},
\end{equation}
where $c_i$ is the medoid of $i$-th class and $c_{x_i}$ is the medoid of the ground-truth class of $x_i$. There are in total $\frac{k(k-1)}{2}$ pairs of class medoids.
Fisher can be calculated in $\mathcal{O}(n^2)$ time.

\textbf{The Revised Silhouette index (RS).}
The commonly used \textit{Silhouette index}~\cite{silhouettes} satisfies scale invariance and isomorphism invariance but fails to satisfy granularity consistency.
We present a revised version of the silhouette index that satisfies all desired properties:
\begin{equation}
    \text{RS}(\mathcal{S}, d) = \frac{1}{n} \sum_{i=1}^n \frac{b(x_i)}{a(x_i)},
\end{equation}
where $a(x_i) = \frac{1}{|C|-1} \sum_{x \in C, x \neq x_i} d(x_i, x)$ is the mean intra-class distance between $x_i$ and other samples in the same class $C$ that $x_i$ belongs to, and $b(x_i)$ is the mean distance between $x_i$ and samples from the nearest class that $x_i$ is not belonging to.
Formally, $b(x_i) = \min_{C' \neq C} d(x_i, C')$, where $d(x_i, C') = \frac{1}{|C'|} \sum_{x \in C'} d(x_i, x)$.
RS can be calculated in $\mathcal{O}(n^2)$ time.

\textbf{The Revised Silhouette with Medoids index (RSM).}
To make the computation of RS faster, we use the medoid to represent each class:
\begin{equation}
    \text{RSM}(\mathcal{S}, d) = \frac{1}{n} \sum_{i=1}^n \frac{d(x_i, c_{x_i}')}{d(x_i, c_{x_i})},
\end{equation}
where $c_{x_i}$ is the medoid of the ground-truth class of $x_i$ and $c_{x_i}'$ is the medoid of the nearest class other than the ground-truth class of $x_i$ (\ie, $c_{x_i}' = \arg \min_{c, c \neq c_{x_i}} d(c, x_i)$, $c$ is the medoid of a class).
RSM can be calculated in $\mathcal{O}(nk)$ time, where $k$ is the number of classes.

\textbf{The Ranking index (Rank).}
Inspired by the Precision-Recall curve and mean Average Precision used in information retrieval, we define the ranking index as:
\begin{equation}
    \text{Rank}(\mathcal{S}, d) = 1 - \frac{1}{n-1} \sum_{i=1}^n \Big(1 - \sum_{j=1}^{|R_i|} \frac{j}{R_{ij}} \Big),
\end{equation}
where $R_i$ is the list of ranks for samples in the same class of $x_i$, sorted by their distance between $x_i$.
Specifically, $\sum_{j=1}^{|R_i|} \frac{j}{R_{ij}}$ is the Average Precision (\ie, the area under Precision-Recall curve) used in information retrieval.
The term $\frac{1}{n - 1}$ is to make sure the range of $\text{Rank}(\mathcal{S}, d) \in [0,1]$.
Rank can be calculated in $\mathcal{O}(n^2 \log n)$ time.

\textbf{The Ranking with Medoids index (RankM).}
Similar to the case of RSM, we use the medoid to present each class for faster computation of Rank index.
The ranking with medoids index is defined as:
\begin{equation}
    \text{RankM}(\mathcal{S}, d) = 1 - \frac{k}{n(k-1)} \sum_{i=1}^n \Big(1 - \frac{1}{R_{ic}} \Big),
\end{equation}
where $R_{ic}$ is the rank of the $x_i$'s class medoid among all class medoids and $k$ is the number of classes.
The term $\frac{k}{n(k-1)}$ is to make sure the range of $\text{RankM}(\mathcal{S}, d) \in [0,1]$.
RankM can be calculated in $\mathcal{O}(nk\log k)$ time.

\begin{proposition}[]
\label{prop:1}
Dataset granularity measures Fisher, RS, RSM, Rank and RankM satisfy granularity consistency, isomorphism invariance and scale invariance.
\end{proposition}

\begin{proof}
The proof is given in Appendix D.
\end{proof}

Instead of the aforementioned measures, there exist other measures that satisfy all desired properties, including the Baker-Hubert Gamma index (BHG)~\cite{baker1975measuring,goodman1979measures} for rank correlation and the C index~\cite{hubert1976general} for clustering quality assessment.
However, the time complexity of BHG and the C index is much higher ($\mathcal{O}(n^4)$ for BHG and $\mathcal{O}(n^3)$ for the C index) and thus are not scalable on large datasets.
For a more comprehensive analysis of other dataset granularity measures and proofs showing that they satisfy our proposed properties, please refer to Appendix E.

\begin{figure}[t]
\begin{center}
\includegraphics[width=0.9\columnwidth]{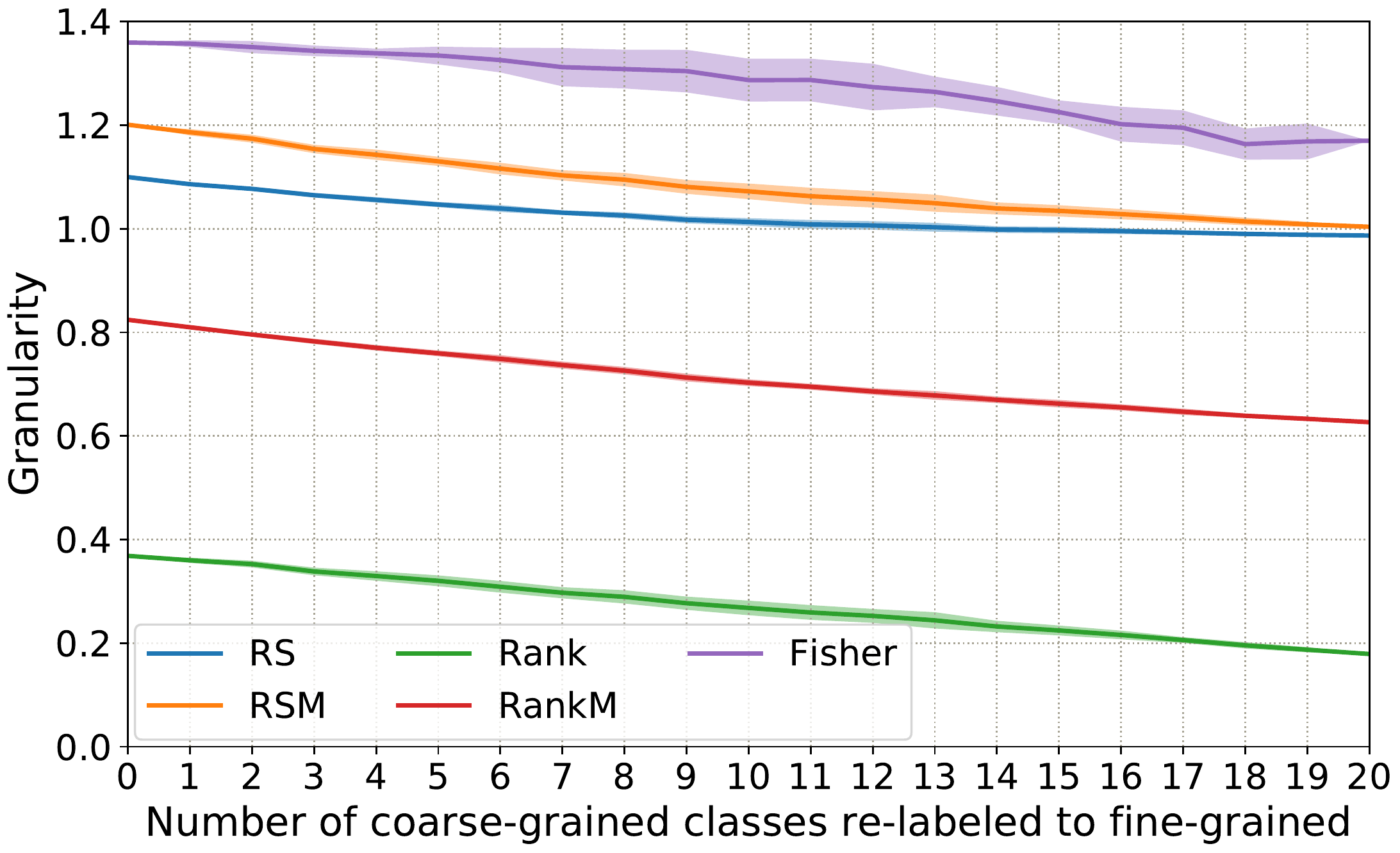}
\end{center}
\caption{Dataset granularity measures RS, RSM, Rank and RankM on CIFAR-100 dataset with increasing number of coarse-grained classes re-labeled to fine-grained. The solid line denotes mean and the shaded region represents standard deviation.}
\label{fig:eval-cifar}
\vspace{-2mm}
\end{figure}

\subsection{Assessment of Granularity Measures}

We evaluate the quality of granularity measure in the context of fine-grained class discovery. Consider a dataset $\mathcal{S}$ annotated with coarse-grained labels and a distance function $d$ as the Euclidean distance between features extracted from a deep network, when we re-label samples from a coarse-grained class with fine-grained sub-classes, the granularity $g(\mathcal{S}, d)$ should be higher. We build the oracle of relative dataset granularity based on such information and use it to evaluate granularity measures.

We use CIFAR-100 dataset~\cite{cifar} with 20 coarse-grained classes, each containing 5 fine-grained classes.
We train a ResNet-20~\cite{resnet} using labels of 20 coarse-grained classes on the CIFAR-100 training set and then extract features on the test set.
When calculating granularity scores, we gradually re-label each coarse-grained class until all 20 coarse-grained classes are re-labeled into 100 fine-grained classes.
To capture the randomness of re-labeling order, we shuffle the classes and repeat the experiments 100 times.
Figure~\ref{fig:eval-cifar} shows the granularity score on CIFAR-100 in terms of how many coarse-grained classes have been re-labeled.
Consistent with the oracle, all granularity measures except Fisher are monotonically decreasing.
Based on results in Figure~\ref{fig:eval-cifar}, RankM has a smoother change of granularity score, less variance and relatively low time complexity compared with others.
Therefore, we choose to use \textit{the Ranking with Medoids index (RankM)} as the dataset granularity measure.
Unless otherwise stated, we use the term ``granularity'' of a dataset $\mathcal{S}$ and a distance function $d$ to refer to $\text{RankM}(\mathcal{S}, d)$ in the rest of our paper. 
Note that RankM is only a local optimal among our proposals.
Other instances of dataset granularity measures that satify all of our disired properties could also be used for measuring granularity, and we leave finding a more optimal one as future work.

\begin{table}[t]
\small
\begin{center}
\begin{tabular}{ l|r|r } 
\Xhline{1.0pt}
\textbf{Dataset} & \textbf{\# class} & \textbf{\# train$\blacktriangle$ / \# test} \\
\Xhline{1.0pt}
Oxford Flowers-102~\cite{flower_102} & 102 & 2,040 / 6,149 \\
CUB200 Birds~\cite{cub200}   & 200 & 5,994 / 5,794 \\
FGVC Aircraft~\cite{airplane}       & 100 & 6,667 / 3,333 \\
Stanford Cars~\cite{stanford_car}  & 196 & 8,144 / 8,041 \\
Stanford Dogs~\cite{stanford_dog}  & 120 & 12,000 / 8,580 \\
NABirds~\cite{nabirds}        & 555 & 23,929 / 24,633 \\
CIFAR-10~\cite{cifar}         & 10 & 50,000 / 10,000 \\
CIFAR-100~\cite{cifar}         & 100 & 50,000 / 10,000 \\
Food-101~\cite{food101}        & 101 & 75,750 / 25,250 \\
\Xhline{1.0pt}
iNaturalist-2017~\cite{inaturalist} & 5,089 & 579,184 / 95,986 \\
ImageNet~\cite{imagenet,ilsvrc} & 1,000 & 1,281,167 / 50,000 \\
\Xhline{1.0pt}
\end{tabular}
\end{center}
\caption{Datasets used in experiments, where iNaturalist-2017 and ImageNet are used to pre-train models for feature extraction.}
\label{tab:dataset}
\vspace{-2mm}
\end{table}

\subsection{Distance Function}
\label{sec:dist_func}
In this section, we use the selected measure, \textit{RankM}, to study the effect of distance function using features extracted from pre-trained deep networks.
We conduct controlled experiments and analyze results from two perspectives: the absolute granularity value and the relative granularity order across different datasets.


To understand how sensitive the dataset granularity is to different choices of distance functions, we fix the pre-trained dataset as ImageNet and perform experiments using deep networks with different architectures including ResNet~\cite{resnet}, DenseNet~\cite{densenet} and Inception-V3~\cite{inception-v3}. The 7 datasets we measured are Oxford Flowers-102~\cite{flower_102}, CUB200-2011 Birds~\cite{cub200}, FGVC Aircraft~\cite{airplane}, Stanford Cars~\cite{stanford_car}, Stanford Dogs~\cite{stanford_dog}, and NABirds~\cite{nabirds}, Food-101~\cite{food101} (see Table~\ref{tab:dataset} for more details). Dataset granularity is calculated using features extracted from the test set.

As Figure~\ref{fig:g-network} shows, the absolute granularity value of a dataset is senstitive to model architecture. For example, DenseNet generally yields higher granularity compared with ResNet and Inception-V3. However, the relative order of granularity across datasets is robust (\eg, Aircraft always has the lowest granularity while Stanford-Dogs the highest).



\begin{figure}[t]
\begin{center}
\includegraphics[width=0.9\columnwidth]{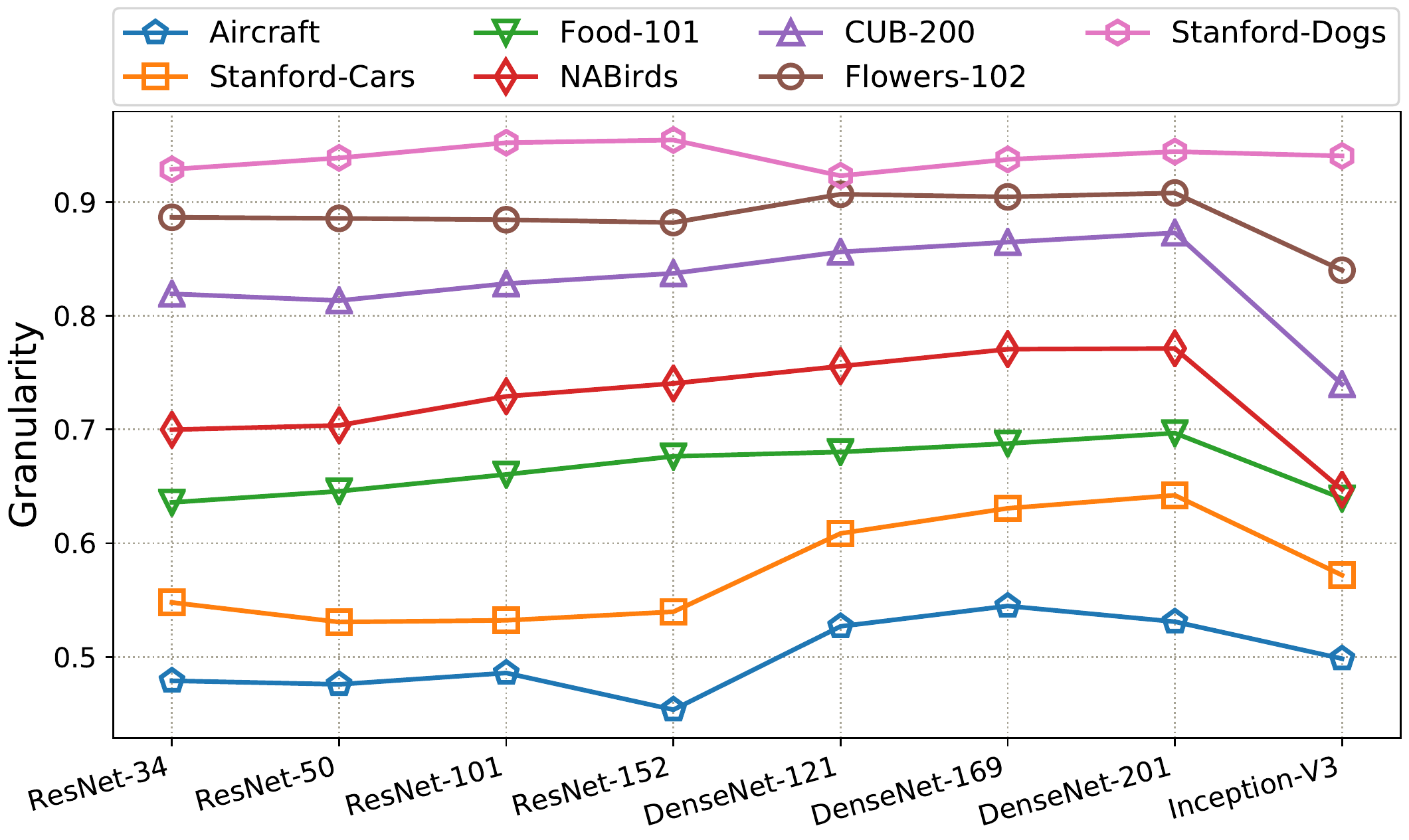}
\end{center}
\caption{Dataset granularity using features extracted from different networks pre-trained on ImageNet.}
\label{fig:g-network}
\vspace{-2mm}
\end{figure}

\section{Bitter and Sweet Tasters}
The conclusion in Section~\ref{sec:dist_func} simplifies our experiment setting without loss of generality: to compare granularity across datasets, in the following sections we use Euclidean distance on features extracted from a ResNet-50 pre-trained on ImageNet as our distance function.
On a given dataset, we aim to find bitter and sweet tasters containing different set of classes with drastically different granularity.

A dataset with $n$ classes has $2^n$ subsets containing different sets of classes.
Enumerating all subsets and measure their granularity to find the bitter and sweet taster is intractable.
We therefore propose a simplified strategy to approximate the optimal solution.
First, we calculate the pairwise granularity between all pairs of classes in the dataset.
Then, to construct the sweet taster, we start with the pair of classes with the largest granularity (most coarse-grained), and iteratively add the class with maximum averaged pairwise granularity between already included classes, until the taster contains the target number of classes.
To construct the bitter taster, we start with a seed set of classes that contain pairs with smallest pairwise granularity, and iteratively add the class with the smallest minimum (instead of the smallest averaged used in sweet taster) pairwise granularity between already included classes, until the taster contains the target number of classes.
The reason to have a seed set for the bitter taster is that we want to include a diverse set of clusters (\eg, a cluster of sparrows and a cluster of warblers in CUB-200). A seed set provides a good coverage of clusters in the dataset.
We set the size of seed set to be $10\%$ of the total number of classes. The size of bitter and sweet tasters are both set to $25\%$ of the total number of classes.

Figure~\ref{fig:overview} shows 10 classes from the bitter and sweet tasters respectively constructed on CUB-200 dataset.
We also tried to increase the size of a sweet taster of CUB-200 by adding classes until its granularity matches CIFAR-10's.
It ended up containing 46 classes, meaning a subset of 46 out of 200 classes in widely considered fine-grained dataset CUB-200 is in fact only as ``fine-grained'' as CIFAR-10.


\section{Experiments}


Equipped with the ability to generate \bitter{} and \sweet{} \taster{} of a given labeled dataset, we study the relation between dataset granularity and model performance on tasks that are being actively studied in computer vision, including image recognition, transfer learning, few-show learning and adversarial attack. For each task, we follow the conventional practices and adopt widely-used benchmarks implemented or reproduced in PyTorch.
Code, data and pre-trained models are available at: {\url{https://github.com/richardaecn/dataset-granularity}}.

\begin{figure}[t]
\begin{center}
\includegraphics[width=0.95\columnwidth]{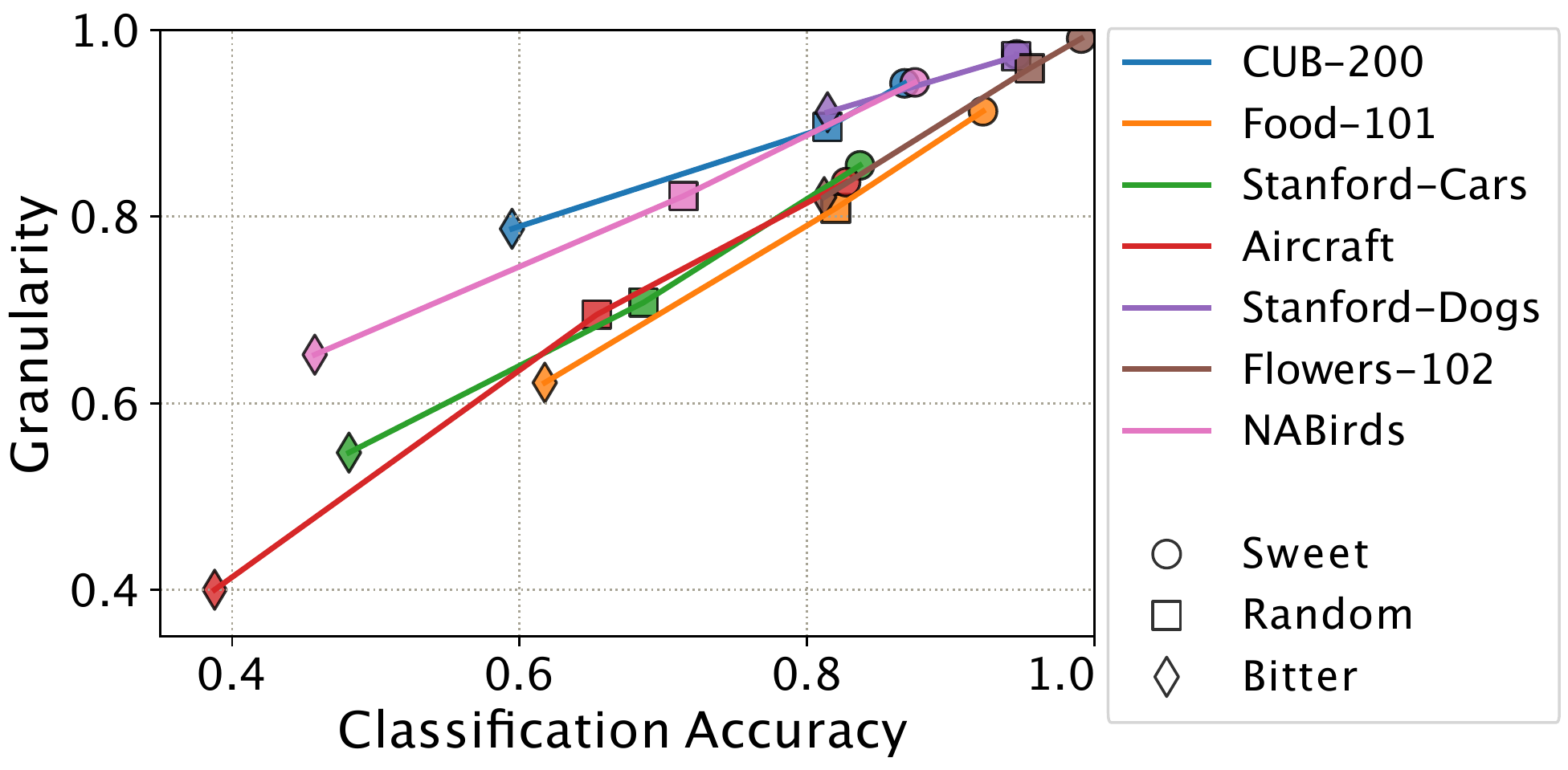}
\end{center}
\caption{Classification accuracy \vs granularity on different tasters. Each dataset is represented by line segments with the same color. Markers with different shapes represent different tasters.}
\label{fig:classification}
\end{figure}

\subsection{Granularity versus Classification}

Image classification is a classic task in computer vision.
We perform experiments on \sweet{} and \bitter{} \tasters{} constructed on 7 datasets by finetuning a ResNet-50 pre-trained on ImageNet for image classification. 
To make sure the distance functions are the same across all datasets for comparing dataset granularity, we freeze the parameters in the network and only update the classifier during training.
Figure~\ref{fig:classification} shows the classification accuracy \vs granularity on all tasters.
The performance on bitter taster is significantly worse than sweet taster, suggesting classification on fine-grained dataset is more difficult compared with coarse-grained dataset.
Therefore, the bitter taster can be easily used for benchmarking image classification algorithms in a stricter and fine-grained manner, especially for datasets that are saturated for benchmarking purposes.
To understand the relation between dataset granularity and classification accuracy, we further conduct experiments on CIFAR and results are shown in Figure~\ref{fig:classification-cifar}.
Dataset granularity correlates quite well with classification accuracy (with Pearson's $\rho = 0.998$), suggesting granularity could be used as an indicator for the classification performance.

\begin{figure}[t]
\begin{center}
\includegraphics[width=0.95\columnwidth]{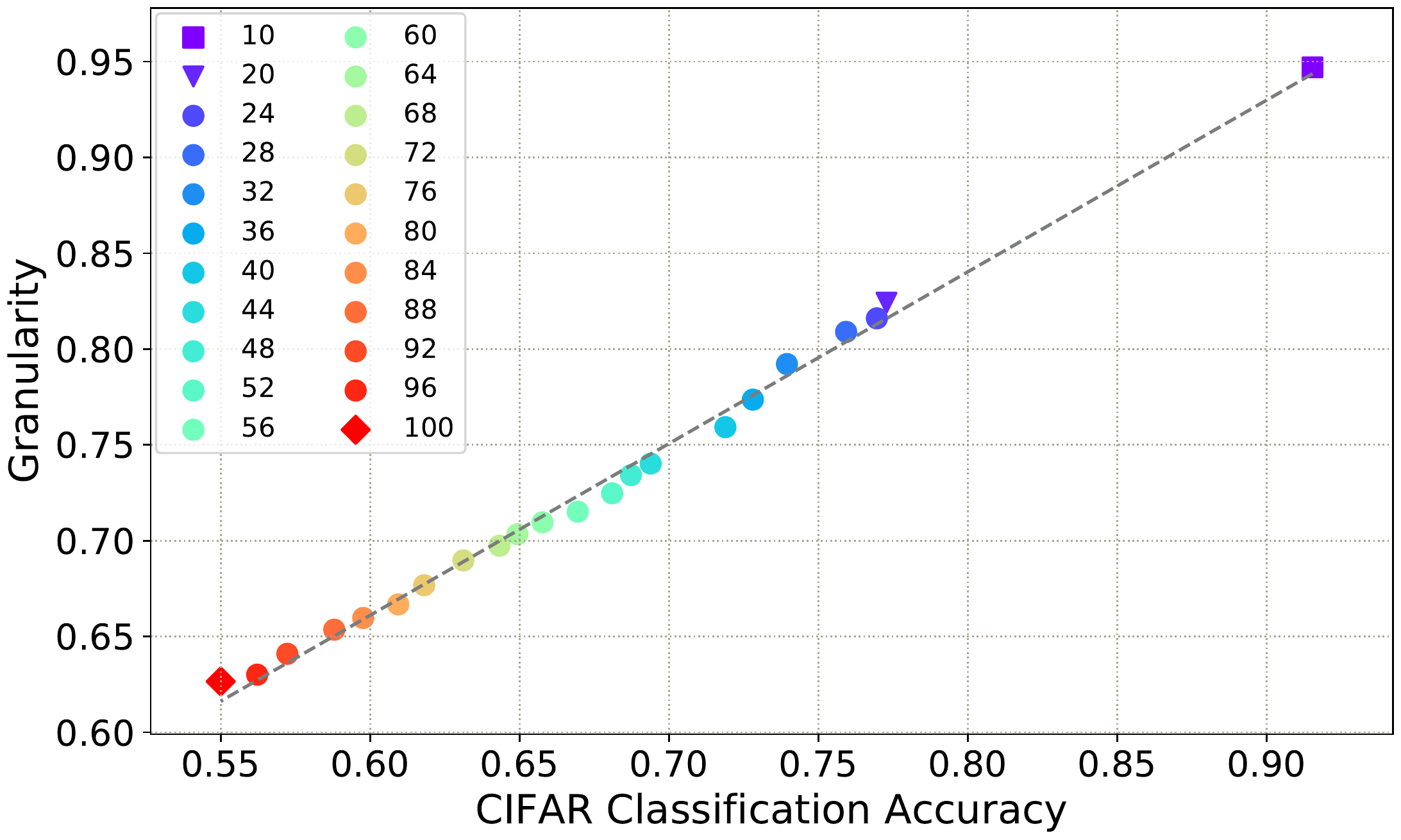}
\end{center}
\caption{Dataset granularity correlates well with classification accuracy on CIFAR. The purple square marker represents CIFAR-10 dataset. Other markers represent the CIFAR-100 with different number of classes (by re-labeling coarse-grained classes with fine-grained), where the purple triangle denotes CIFAR-100 with 20 coarse-grained labels and the red diamond represents CIFAR-100 with 100 fine-grained labels.}
\label{fig:classification-cifar}
\end{figure}

\begin{figure}[t]
\begin{center}
\includegraphics[width=0.9\columnwidth]{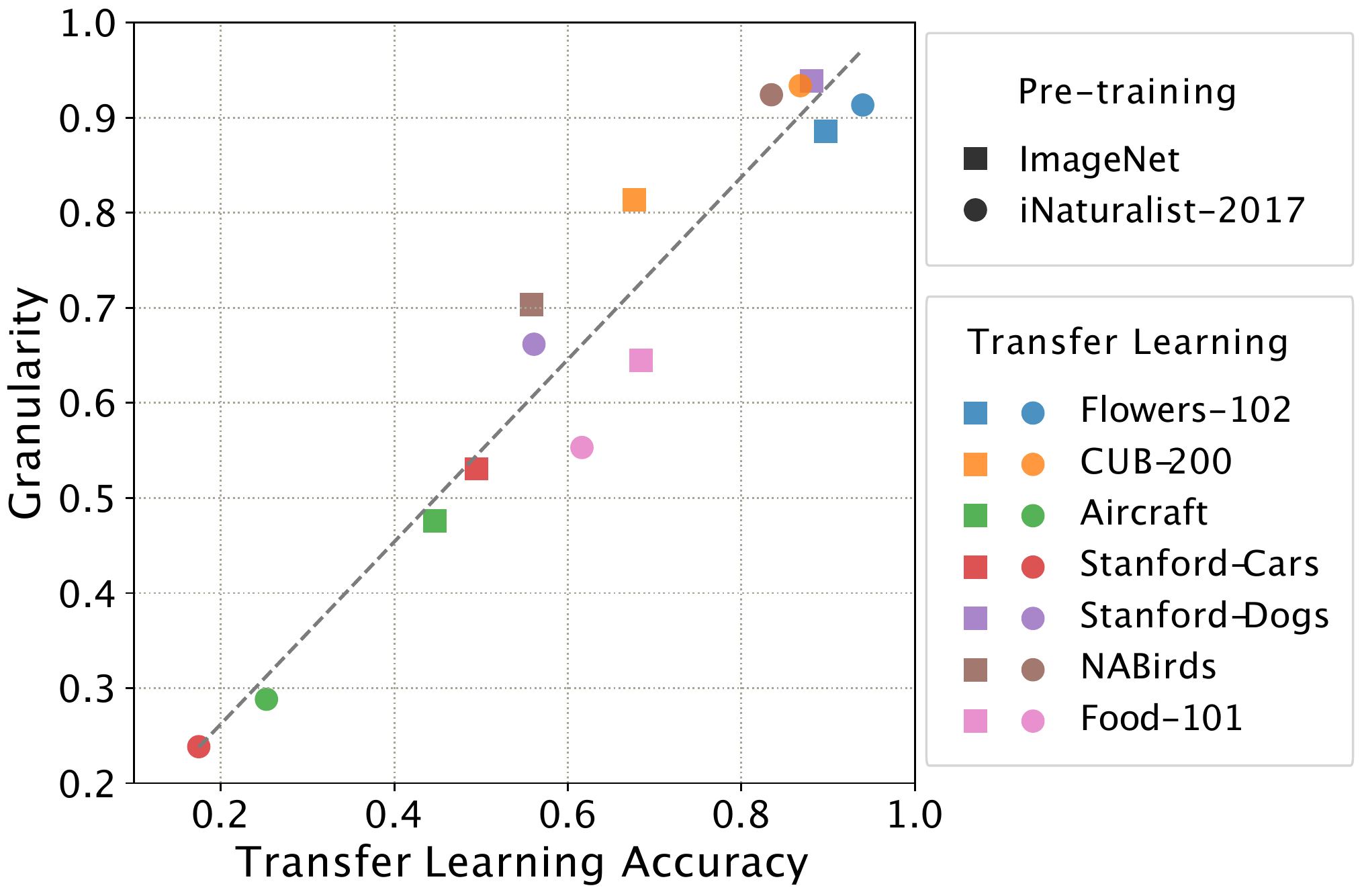}
\end{center}
\caption{Transfer learning performance and granularity on 7 datasets. Each dataset is represented by markers with the same color. We use markers with different shapes to represent pre-training on different datasets.}
\label{fig:transfer}
\vspace{-4mm}
\end{figure}

\begin{figure}[t]
\begin{center}
\subfigure[\scriptsize{Stanford-Dogs - IN (0.951)}]{
\includegraphics[width=0.4\columnwidth]{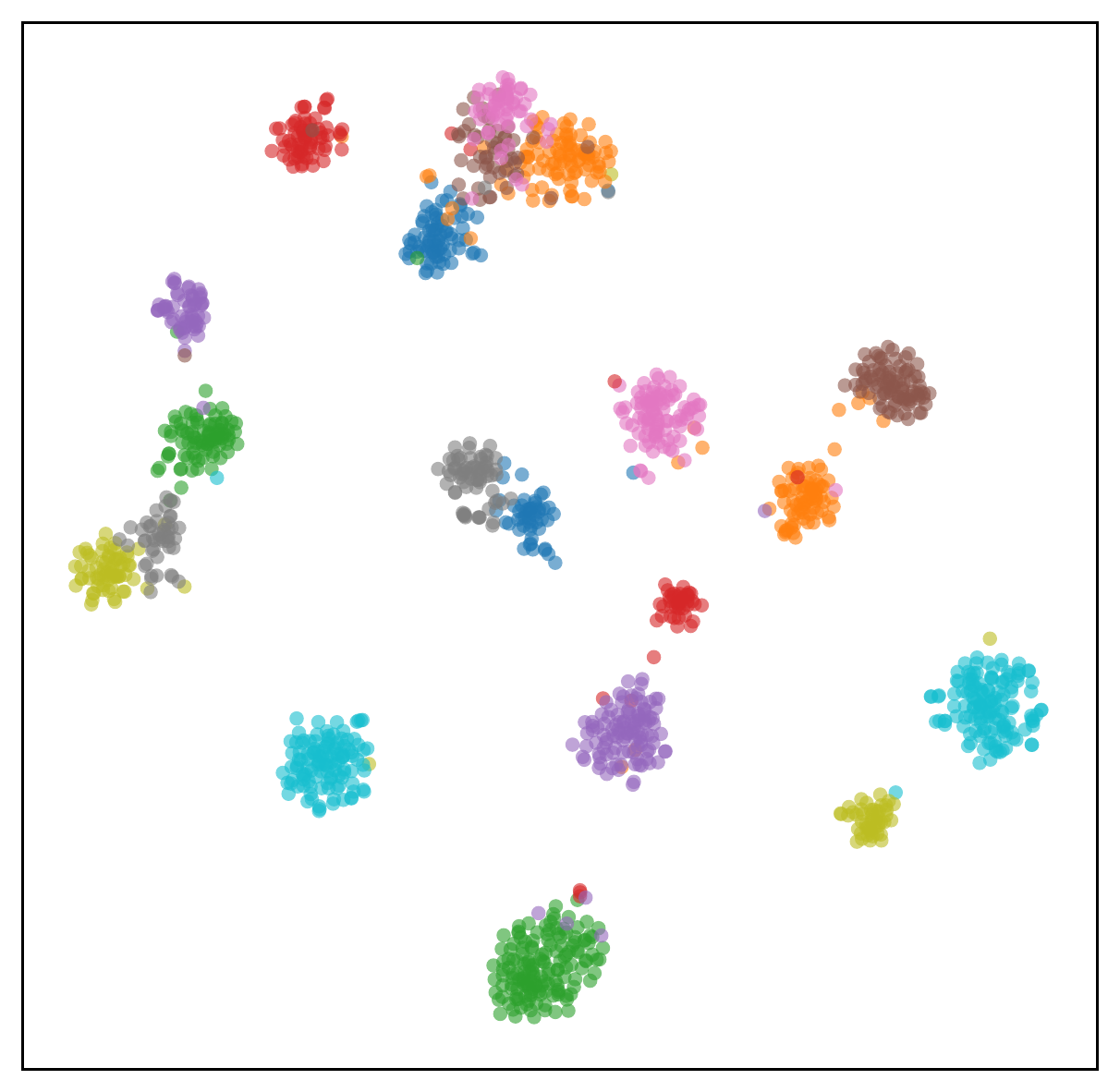}}
\subfigure[\scriptsize{Stanford-Dogs - iNat (0.586)}]{
\includegraphics[width=0.4\columnwidth]{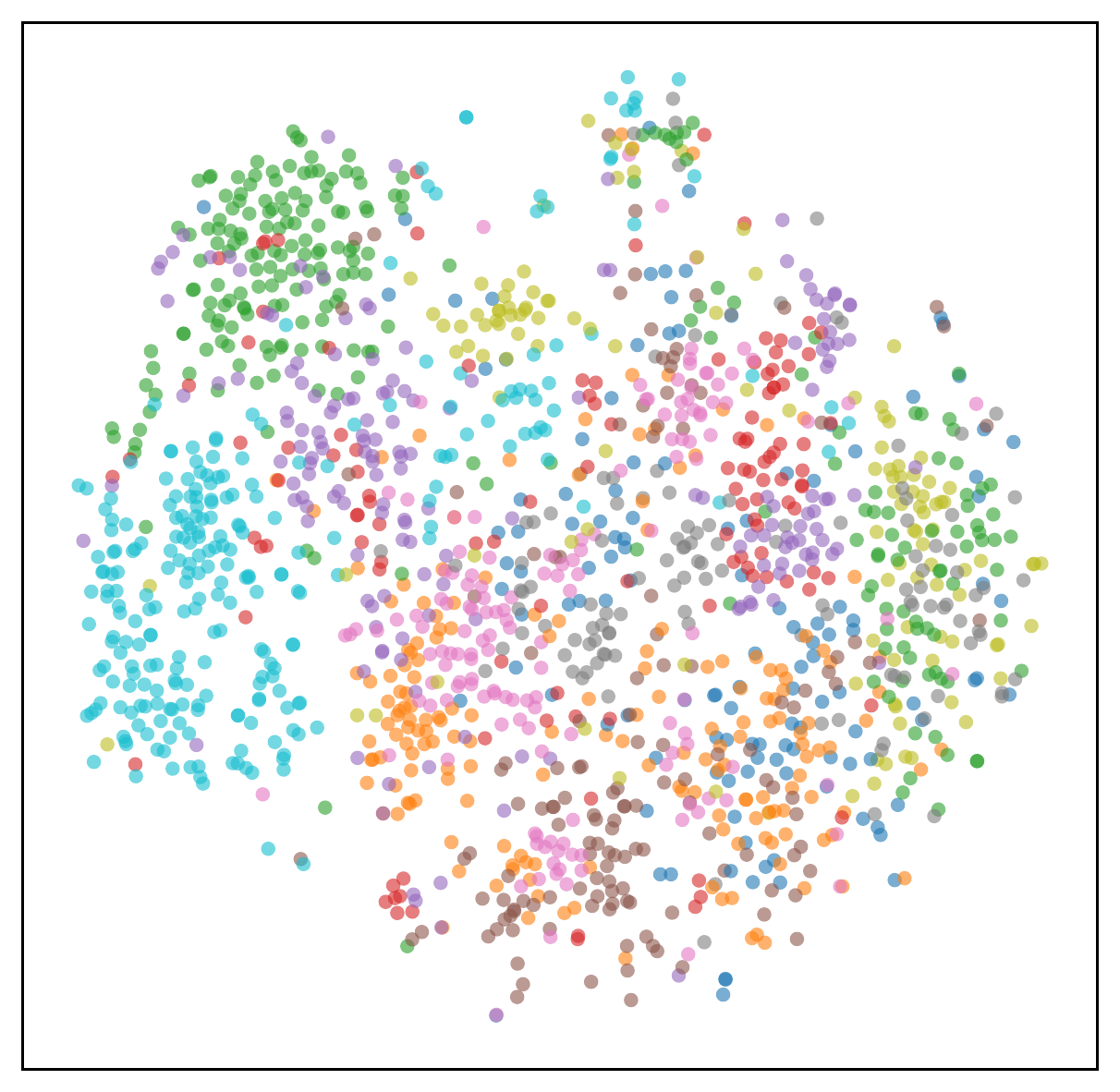}}
\subfigure[\scriptsize{NABirds - IN (0.620)}]{
\includegraphics[width=0.4\columnwidth]{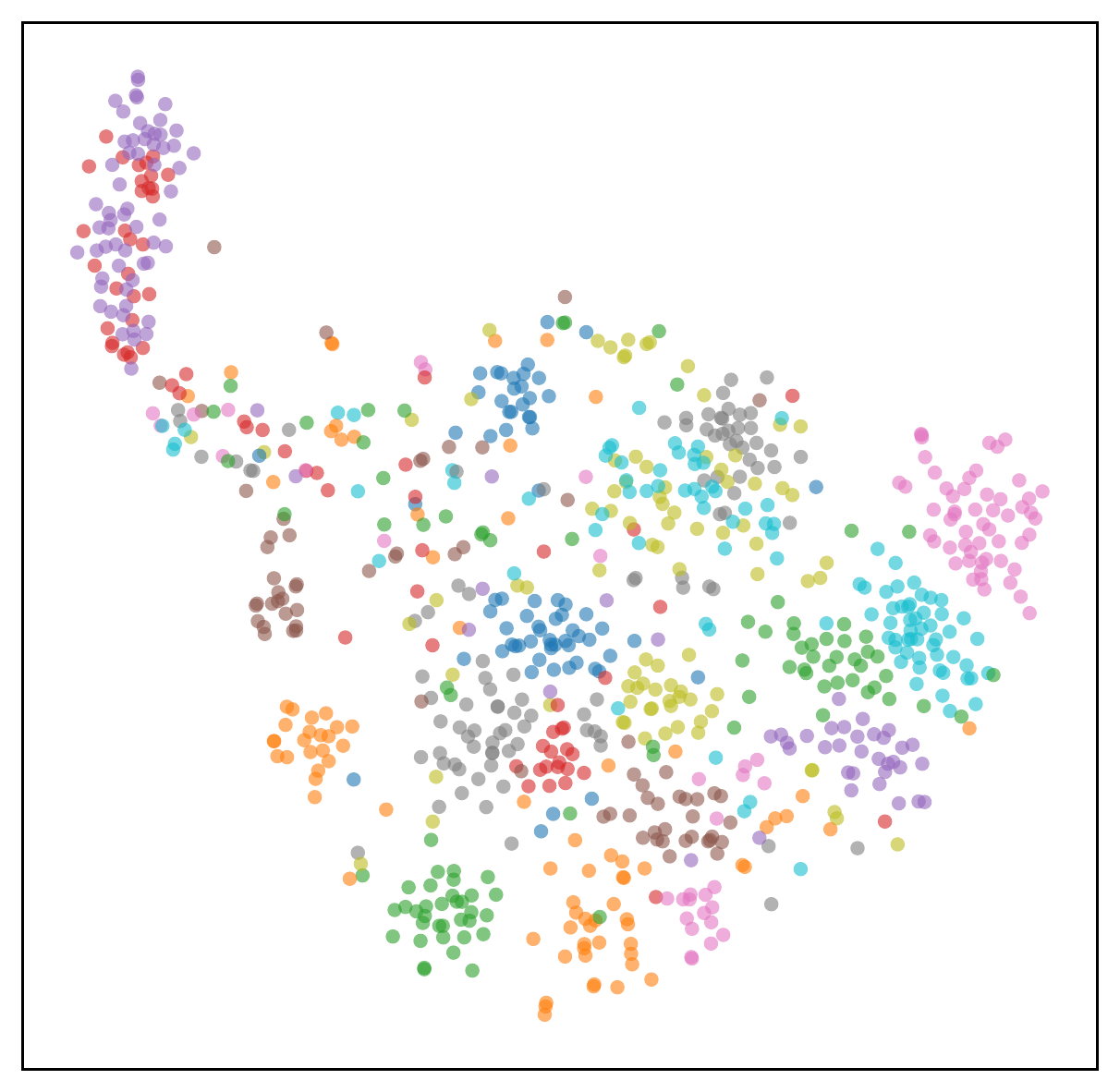}}
\subfigure[\scriptsize{NABirds - iNat (0.925)}]{
\includegraphics[width=0.4\columnwidth]{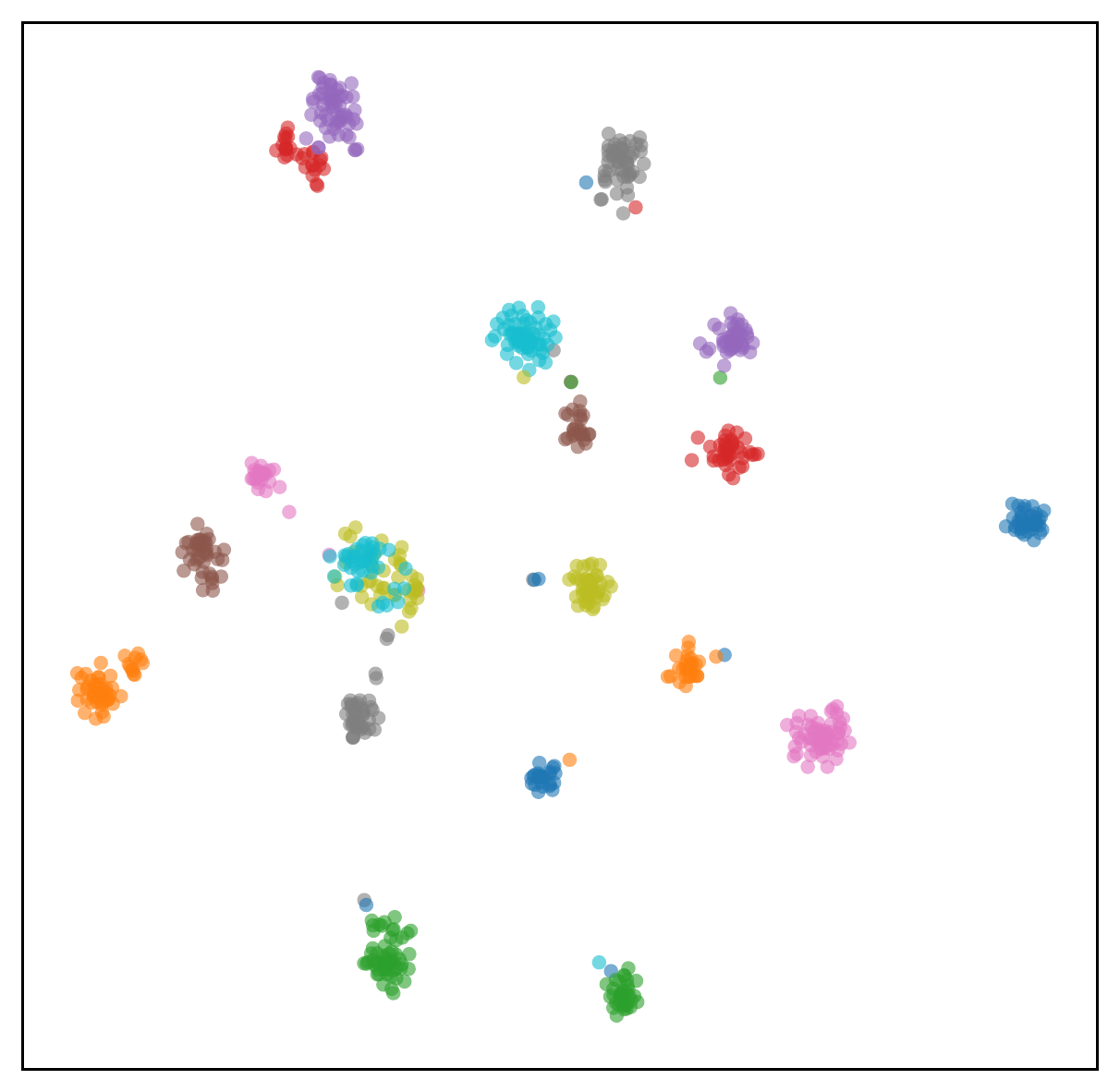}}
\end{center}
\caption{t-SNE embedding of first 20 classes from Stanford-Dogs and NABirds, with ImageNet pre-training (IN) on the left and iNaturalist-2017 (iNat) pre-training on the right. The granularity is shown in the bracket.}
\label{fig:tsne-fgvc}
\vspace{-4mm}
\end{figure}

\subsection{Granularity versus Transferability}
\label{sec:exp_transfer}

Transferring features learned from large-scale datasets to small-scale datasets has been extensively used in our field.
To examine the relationship between dataset granularity and feature transferability, we train ResNet-50 networks on 2 large-scale datasets: ImageNet and iNaturalist-2017. 
Then, we transfer the learned features to 7 datasets via fine-tuning by freezing the network parameters and only update the classifier.
We evaluate the performance of transfer learning by the classification accuracy on the test set.
The dataset granularity is calculated based on test set features.

From results in Figure~\ref{fig:transfer}, we have the following observations: 
(1) In general, similar to image classification results, dataset granularity correlates well with transfer learning performance (with Pearson's $\rho = 0.980$) and could be used as an indicator for the transfer learning performance as well. The more coarse-grained the dataset is (higher granularity), the easier it is to transfer to.
(2) The choice of pre-trained data has a huge effect on the dataset granularity and transfer learning performance. 
For example, when using ImageNet pre-trained features, the granularity of Stanford-Dogs is much higher than using iNaturalist-2017, while the granularity of NABirds is much lower.
This is illustrated qualitatively by t-SNE embeddings in Figure~\ref{fig:tsne-fgvc}.
A key to make a dataset less fine-grained is to use features extracted from a network pre-trained on more appropriate data.
This has been also verified by recent work on domain-specific adaptive transfer learning~\cite{ge2017borrowing,cui2018large,ngiam2018domain}.

\subsection{Granularity versus Few-Shot Learning}
Few-shot learning requires the model to learn features that are generic enough to efficiently represent data in test set that is disjoint from training set.
We follow the conventional settings of 5-way 1-shot experiments as described in~\cite{sung2018learning}, and use 3 benchmark approaches for evaluation: Model-Agnostic Meta-Learning (MAML)~\cite{finn2017model}, Prototypical Networks~\cite{snell2017prototypical}, and Relation Network~\cite{sung2018learning}. 
The granularity of miniImageNet~\cite{ravi2016optimization} is close to that of the entire ImageNet, so to observe the effects of dataset granularity, we extract the \bitter{} and \sweet{} \tasters{} of ImageNet and split it into train, validation and test set in the same fashion as the original miniImageNet, and name them \textit{Sweet miniImageNet} and \textit{Bitter miniImageNet}. To analyze current benchmarks' performance with respect to granularity, we train the models on the original miniImageNet and test them on our \bitter{} and \sweet{} miniImageNet (Bitter-test and Sweet-test in Table~\ref{tab:few-shot}). Then repeat the experiment but use our \tasters{} for both training and testing (Bitter and Sweet in Table~\ref{tab:few-shot}).
Compared with the original split, the test accuracy is lower by more than 10\% on our \bitter{} \taster{}, regardless of the training set. 
This suggests our \bitter{} miniImageNet provides a more rigorous benchmark for few-shot learning. For testing on \bitter{} \taster{}, training on \bitter{} \taster{} does not bring much benefits compared with training on the original training set, but for testing on \sweet{} \taster{}, the benefits are more obvious. Besides, there is no change in performance ranking: networks that perform better on the original miniImageNet also achieve higher accuracy on our \tasters{}, indicating the robustness of using our \tasters{} as benchmarks.

\begin{table}[t]
\begin{center}
 \begin{tabular}{c | c c c } 
 \Xhline{1.0pt}
 & MAML & Prototypical & Relation \\
 \Xhline{1.0pt}
 Original & 44.3 / 48.7$^\dagger$ & 46.6 / 49.4$^\dagger$ & 50.4 / 50.4$^\dagger$\\
 \hline
 Bitter-test & 34.3 & 36.1 & 38.8 \\
 Sweet-test & 50.1 & 54.7 & 58.7 \\
 \hline
 Bitter & 32.6 & 35.0 & 37.4 \\
 Sweet & 57.2 & 61.3 & 58.8 \\
 \Xhline{1.0pt}
 \end{tabular}
\end{center}
\caption{Test accuracies of few-shot learning on miniImageNet and \bitter{} and \sweet{} miniImageNet. $^\dagger$ indicates the reported number in the original paper.}
\label{tab:few-shot}
\vspace{-2mm}
\end{table}

\subsection{Granularity versus Adversarial Attack}
Recent studied have shown that neural networks are vulnerable to adversarial examples, malicious inputs that are constructed by adding to a natural image minimal perturbation inperceptible to human eyes. Similar to few-shot learning, adversarial attack has been mostly tested on generic datasets like ImageNet and CIFAR. The performance of adversarial attacks with respect to dataset granularity is an interesting topic lacks investigation.

DeepFool~\cite{moosavi2016deepfool} is a benchmark adversarial attack designed to fool the model with least amount of perturbations possible. It proposes to use $\hat{\rho_{adv}}(k) = \frac{1}{\|\mathcal{D}\|} \sum_{x\in{\mathcal{D}}} \frac{\|\mathbf{\hat{r} (x)}\|_{p}}{\|\mathbf{x}\|_{p}}$ to measure the robustness of a classifier $k$, wherein given an image $\mathbf{x}$ from test set  $\mathcal{D}$, $\hat{\mathbf{r}}(x)$ is the estimated minimal perturbation obtained
using DeepFool, measured by distance function $\ell_p$ ($p \in [1,\infty)$ and we follow the original paper to use $p=2$). Intuitively, given a classifier, \padv{} is the average minimum perturbation (normalized by the magnitude of the image) that images required for successful attacks.
We apply it in a different way from the original paper: instead of fixing the dataset and comparing across models, we fix the model and compare across datasets. 
We fine-tune a ResNet-50~\cite{resnet} pre-trained on ImageNet~\cite{imagenet} on the \bitter{}, \sweet{} and random \tasters{} of the same size (50 classes) constructed from CUB-200 respectively, and calculate \padv{} for each as the target model.
Then \padv{} can be interpreted as the estimated average minimal distance of a sample to its closest class boundary. Therefore, the fooling ratio needs to be 100\% to measure \padv{}, and default value of maximal number of iterations in the paper, 50, can achieve it already. We find the \padv{} of \sweet{} \taster{} is more than $2\times$ larger than that of \bitter{} \taster{}, and that of random \taster{} lies in between. This supports our understanding that there is an positive correlation between the dataset granularity and inter-class distance. As dataset granularity decreases, the average margin between neighboring classes decreases, therefore the adversarial attack is easier to succeed. To verify our hypothesis, we set maximal number of iteration to 5 as a stricter constraint, and the second row of Table~\ref{tab:adv_attack} shows the model accuracy after attack. As we expected, the lower the dataset granularity, the lower the accuracy and the easier the attack.

\begin{table}[t]
\begin{center}
 \begin{tabular}{c | c c c} 
 \Xhline{1.0pt}
 & \multicolumn{3}{c}{CUB-200} \\
 \cline{2-4}
 & Bitter & Sweet & Random \\ 
 \Xhline{1.0pt}
 \padv{} & 0.0095 & 0.0216 & 0.0159\\ 
 Test Accuracy & 1.07\% & 6.94\% & 3.74\%\\ 
 \Xhline{1.0pt}
\end{tabular}
\end{center}
\caption{\padv{} and test accuracy of attacking a ImageNet pre-trained ResNet-50 fine-tuned on the \bitter{}, \sweet{} and random \tasters{} of CUB-200.}
\label{tab:adv_attack}
\vspace{-2mm}
\end{table}

\section{Conclusion}

In this work, as a step toward a precise definition of \textit{fine-grained}, we present an axiomatic framework based on clustering theory for measuring dataset granularity.
We define a dataset granularity measure as a function that maps a labeled dataset and a distance function to a real number and then describe three desirable properties for the measure: granularity consistency, isomorphism invariance, and scale invariance.
We give five qualified examples of granularity measure and choose the Ranking with Medoids index (RankM) based on assessment on a hierarchically labeled CIFAR dataset. With this measure, we conduct experiments to study properties of dataset granularity and how to construct fine-grained and coarse-grained subsets of a dataset.
We also show that dataset granularity correlates well with the performance of classification, transfer learning, few-shot learning and adversarial attack.

We show the granularity of current fine-grained dataset is often brought up by a large set of relatively coarse-grained classes. One major reason is that when building a dataset, we were not clear about how to quantitatively measure and control important properties of the dataset. We believe granularity is only one of these properties, and only provides a preliminary effort. In the future, we would like to: 1) explore additional axioms of dataset granularity (if any) and study the completeness of our framework, 2) find other fundamental dataset properties that cannot be captured by granularity, and 3) apply granularity measure in constructing datasets and benchmarking existing methods.

\vspace{2mm}
\par\noindent\textbf{Acknowledgments:}
We would like to thank Yixuan Li, Bharath Hariharan and Pietro Perona for their insights and helpful discussions.



\newpage
\section*{Appendix A: Simulation}
To better understand the characteristics of granularity measure candidates, we perform a study using simulated datasets with 2 classes, each class has $1000$ samples.
Samples for each class are drawn from a 2-dimensional multivariate normal distributions with unit covariance matrix and different means in horizontal axis.
Formally, $X_1 \sim \mathcal{N}(\Sigma, \mu_1)$ and $X_2 \sim \mathcal{N}(\Sigma, \mu_2)$, where $\Sigma = \begin{psmallmatrix}1 & 0\\ 0 & 1\end{psmallmatrix}$, $\mu_1 = (0, 0)^\top$ and $\mu_2 = (m, 0)^\top$.
Figure~\ref{fig:sim-gaussian} illustrates the simulated datasets with different $m$.
The larger the $m$ is, the more separated the two classes are.

We generate simulated datasets with variant $m$ and calculate their granularity measure scores using Euclidean distance as the distance function.
We repeat the this process $10$ times and show the mean and standard deviation in Figure~\ref{fig:score-gaussian}.
From the figure we can observe that the variance of the scores are small and in general, granularity measures Fisher, RS, RSM, Rank and RankM are all able to give higher scores when two classes are further separated.

\section*{Appendix B: Learning Difficulty}
\label{sec_iccv19:exp_difficulty}

Intuitively, learning on fine-grained data is more difficult compared with coarse-grained data.
To quantitatively understand the relationship between learning difficulty and granularity, we use the training error rate of a linear logistic regression model (LR) trained with deep features as an estimate of training difficulty.

We conduct comprehensive experiments on CIFAR. Specifically, we train a ResNet-20 from scratch on the training set. We then extract features on the test set and use them to train a LR and measure granularity and the training error rate.
Figure~\ref{fig:difficulty-cifar} presents results on CIFAR datasets, including CIFAR-10 and CIFAR-100 with variant number of coarse-grained classes re-labeled as fine-grained.
Further more, we examine the case of binary classification by forming a dataset with two specific classes for each pair of classes in CIFAR-10 and CIFAR-100, results are shown in Figure~\ref{fig:difficulty-cifar-pair}. 
We observe that dataset granularity is highly correlated with learning difficulty.
In addition, we find that machine's perception of granularity makes sense to human: (``cat'', ``dog''), (``otter'', ``seal'') and 
(``boy'', ``girl'') are the most fine-grained pairs of classes and are also relatively difficult to differentiate for human.

Qualitatively, we show t-SNE embeddings~\cite{tsne} of sweet and bitter tasters with drastically different granularities from CIFAR-10 and CIFAR-100 in Figure~\ref{fig:tsne-cifar}.
Classes with high granularity are well separated, whereas classes with low granularity are mixed together.

Other than linear logistic regression, we also tried to use the training error rate of deep networks to indicate the learning difficulty.
However, we find deep networks tend to have close to 0 training error rate on all datasets.
As pointed out by Zhang~\etal~\cite{zhang2016understanding}, deep network can easily fit a training set, especially when the number of parameters exceeds the number of data points as it usually does in practice.
In such scenario, if we want to measure the learning difficulty of a deep network, we might need to seek for other metrics such as the intrinsic dimension~\cite{li2018measuring} of the network.

\begin{figure}[t]
\begin{center}
\subfigure[$m=1$]{
\includegraphics[width=0.45\columnwidth]{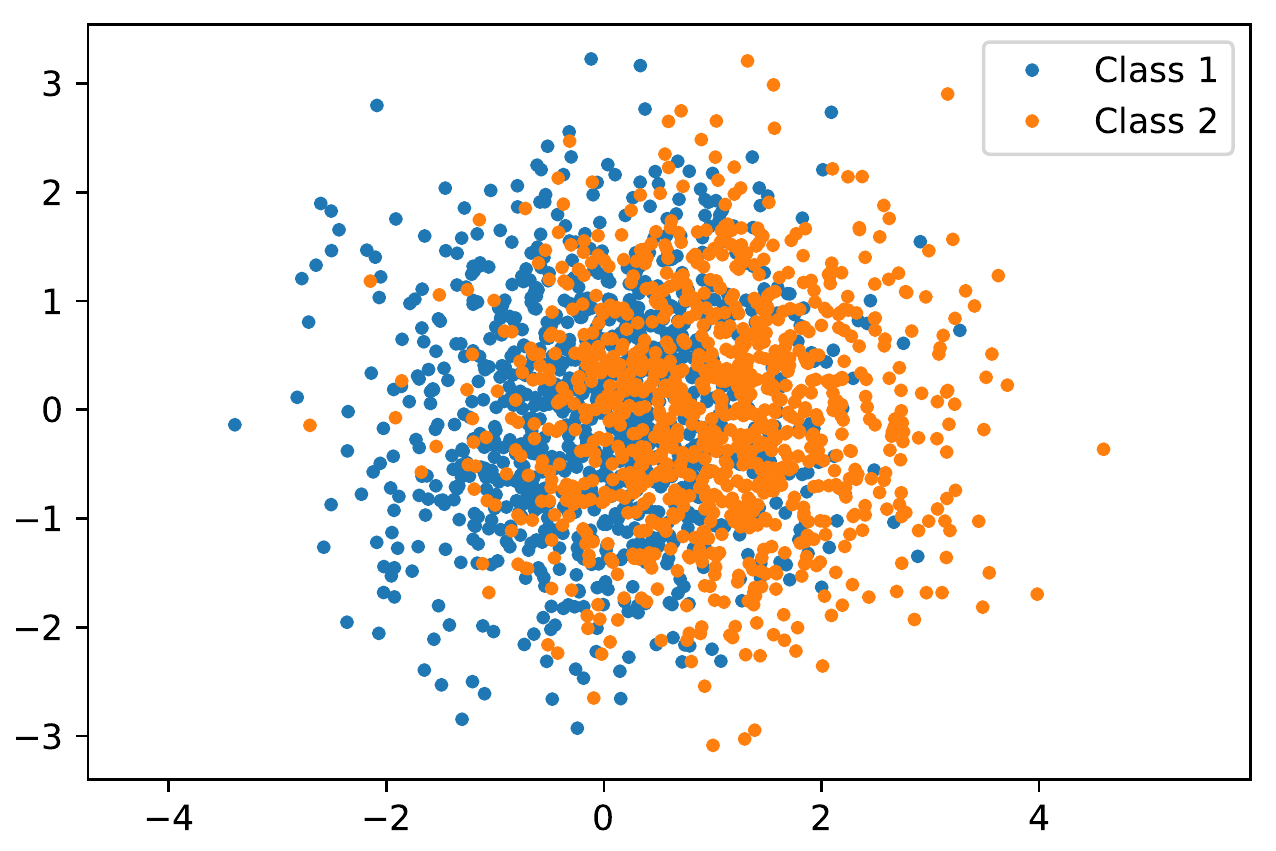}
  \label{fig:sim-1}}
\subfigure[$m=3$]{
\includegraphics[width=0.45\columnwidth]{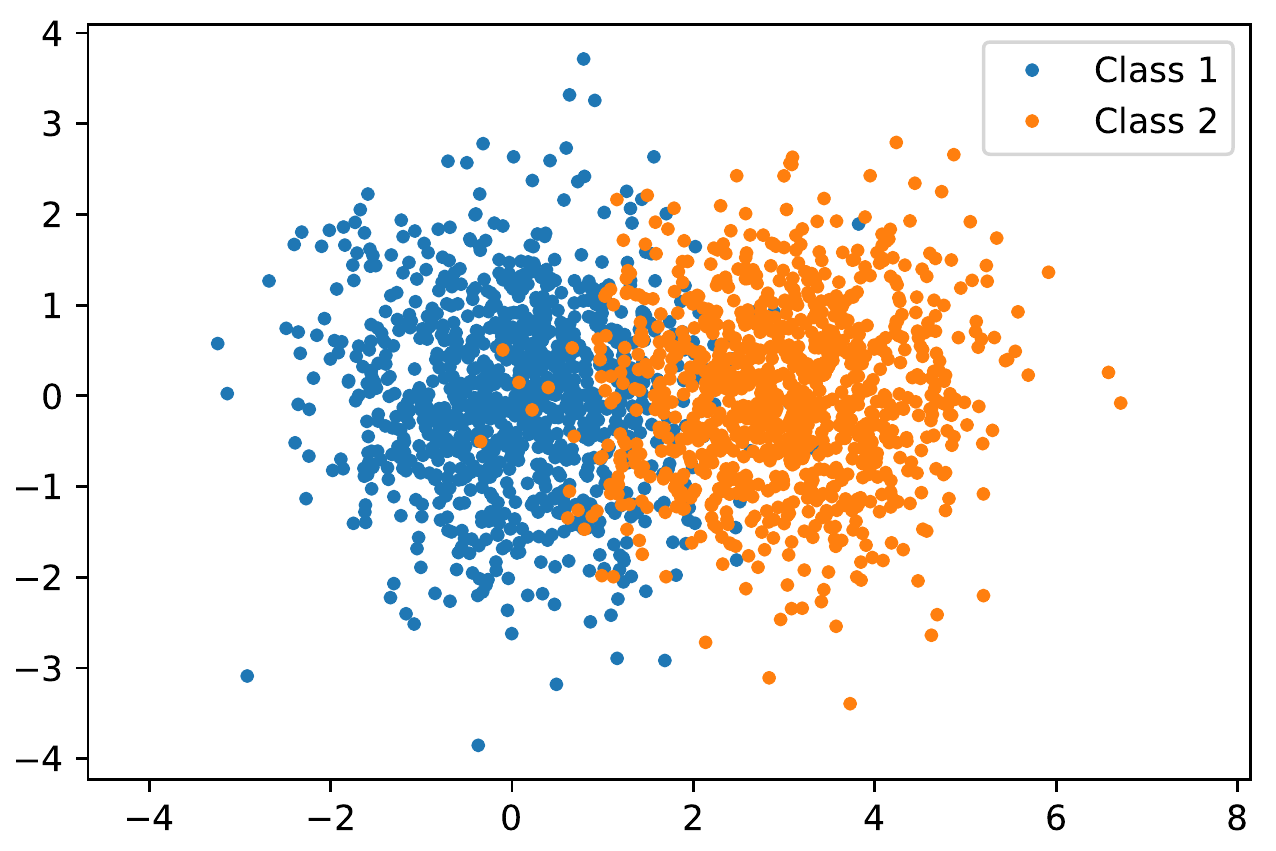}
  \label{fig:sim-3}}
\end{center}
\caption{Examples of simulated datasets of 2 classes. Samples from each class are draw from a 2-dimensional multivariate normal distributions with unit covariance matrix. $m$ denotes the distance between the means of two classes.}
\label{fig:sim-gaussian}
\end{figure}

\begin{figure}[t]
\begin{center}
\includegraphics[width=0.9\columnwidth]{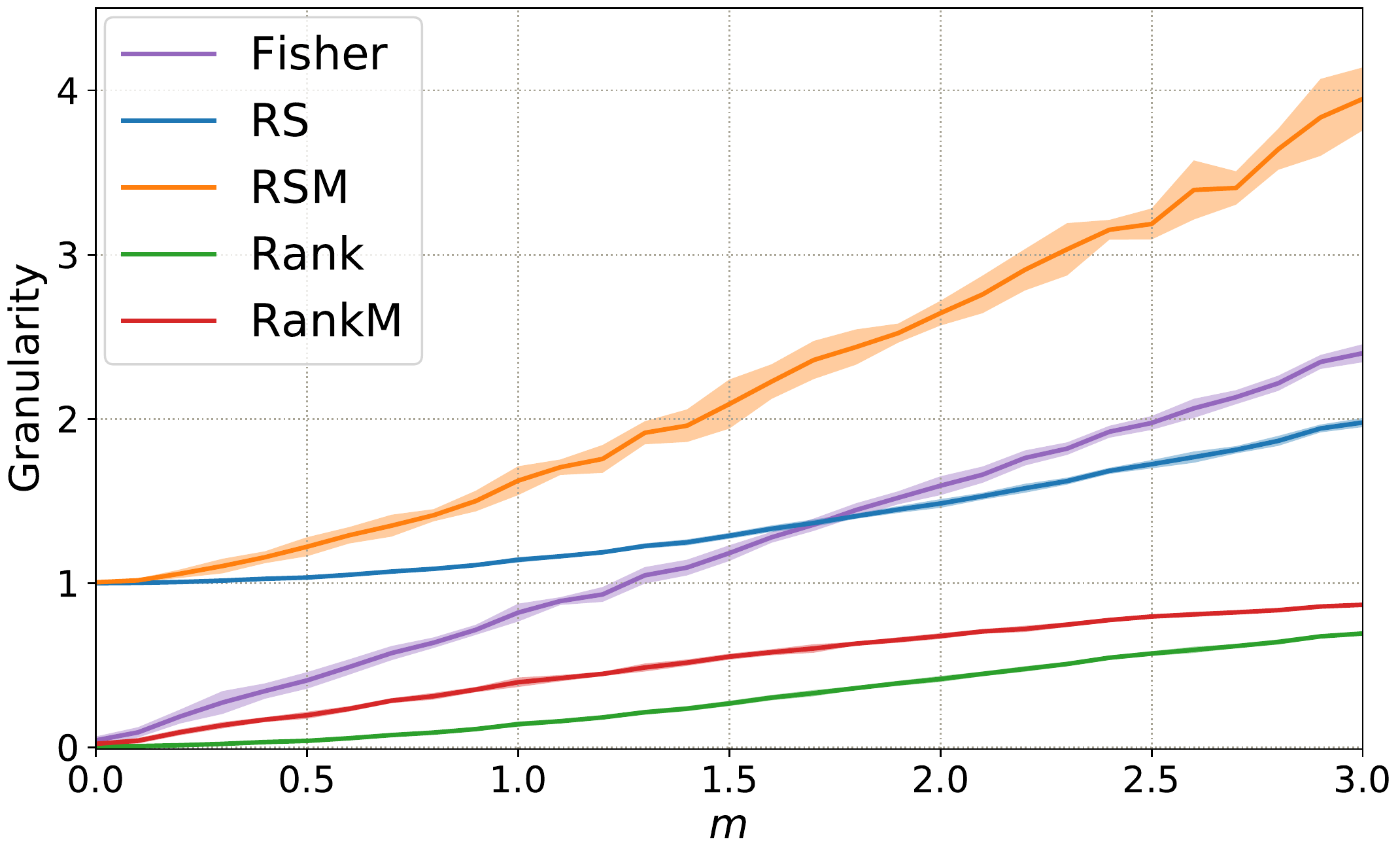}
\end{center}
\caption{Dataset granularity measures Fisher, RS, RSM, Rank and RankM on simulated dataset with variant $m$. The solid line denotes mean and the shaded region represents standard deviation.}
\label{fig:score-gaussian}
\end{figure}

\begin{figure}[t]
\begin{center}
\includegraphics[width=0.75\columnwidth]{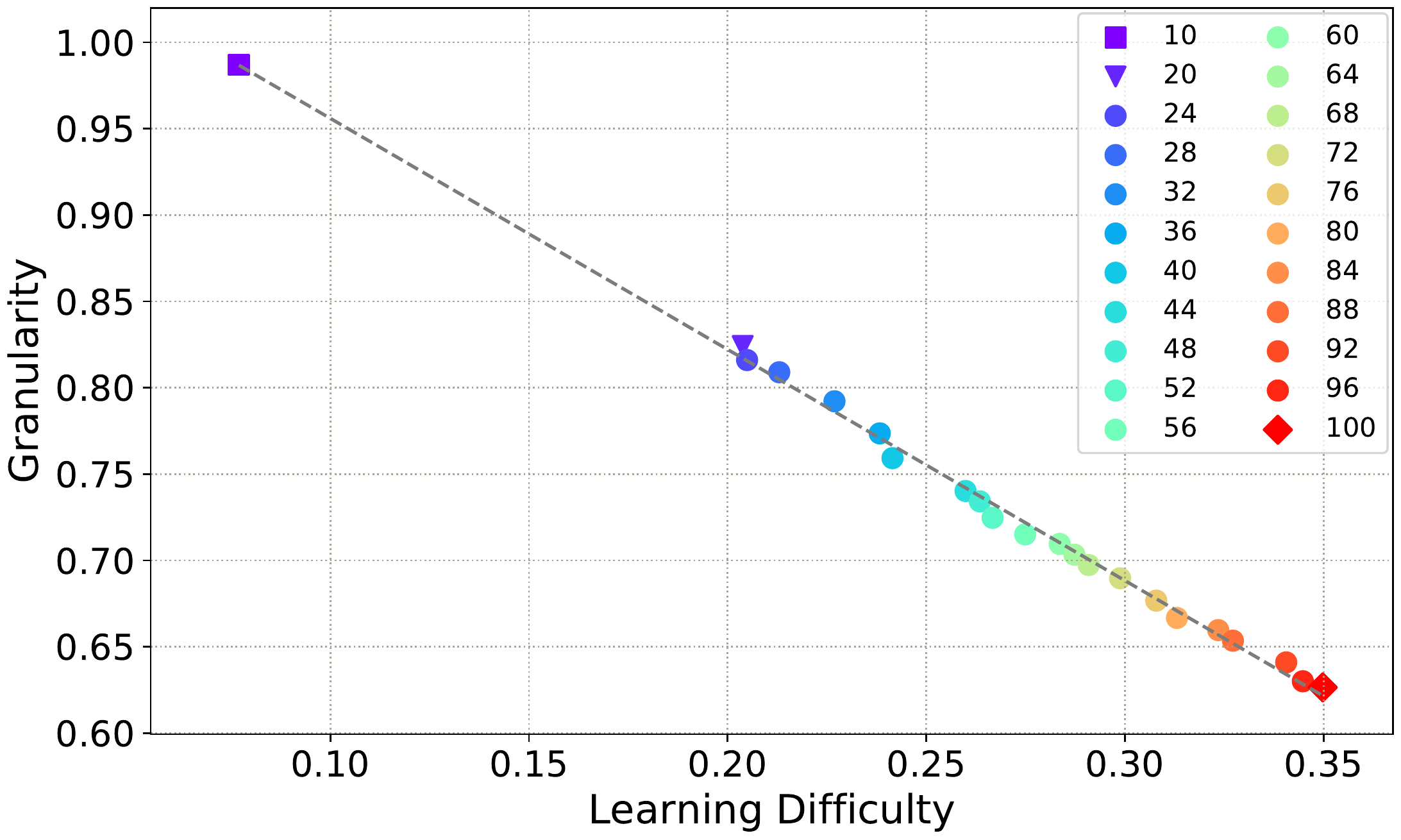}
\end{center}
\caption{Dataset granularity correlates well with difficulty (training error rate using linear logistic regression). The purple square marker represents CIFAR-10 dataset. Other markers represent the CIFAR-100 with different number of classes (by re-labeling coarse-grained classes with fine-grained), where the purple triangle and the red diamond denote CIFAR-100 with 20 and 100 coarse-grained labels respectively.}
\label{fig:difficulty-cifar}
\end{figure}

\begin{figure}[t]
\begin{center}
\subfigure[CIFAR-10 Pairs ($\rho = 0.992$)]{
\includegraphics[width=0.45\columnwidth]{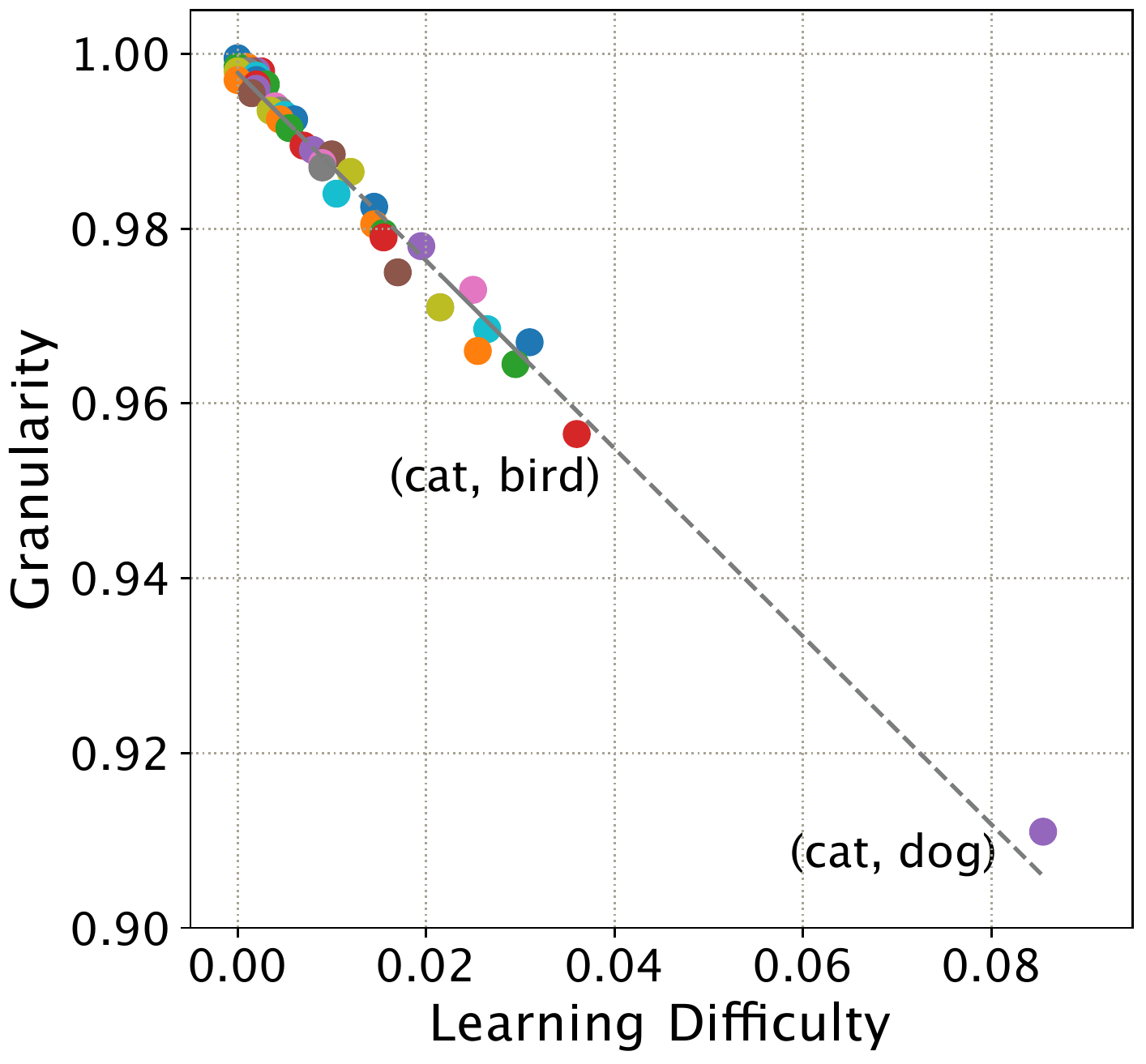}
   \label{fig:difficulty-cifar10-pair}}
\subfigure[CIFAR-100 Pairs ($\rho = 0.958$)]{
\includegraphics[width=0.45\columnwidth]{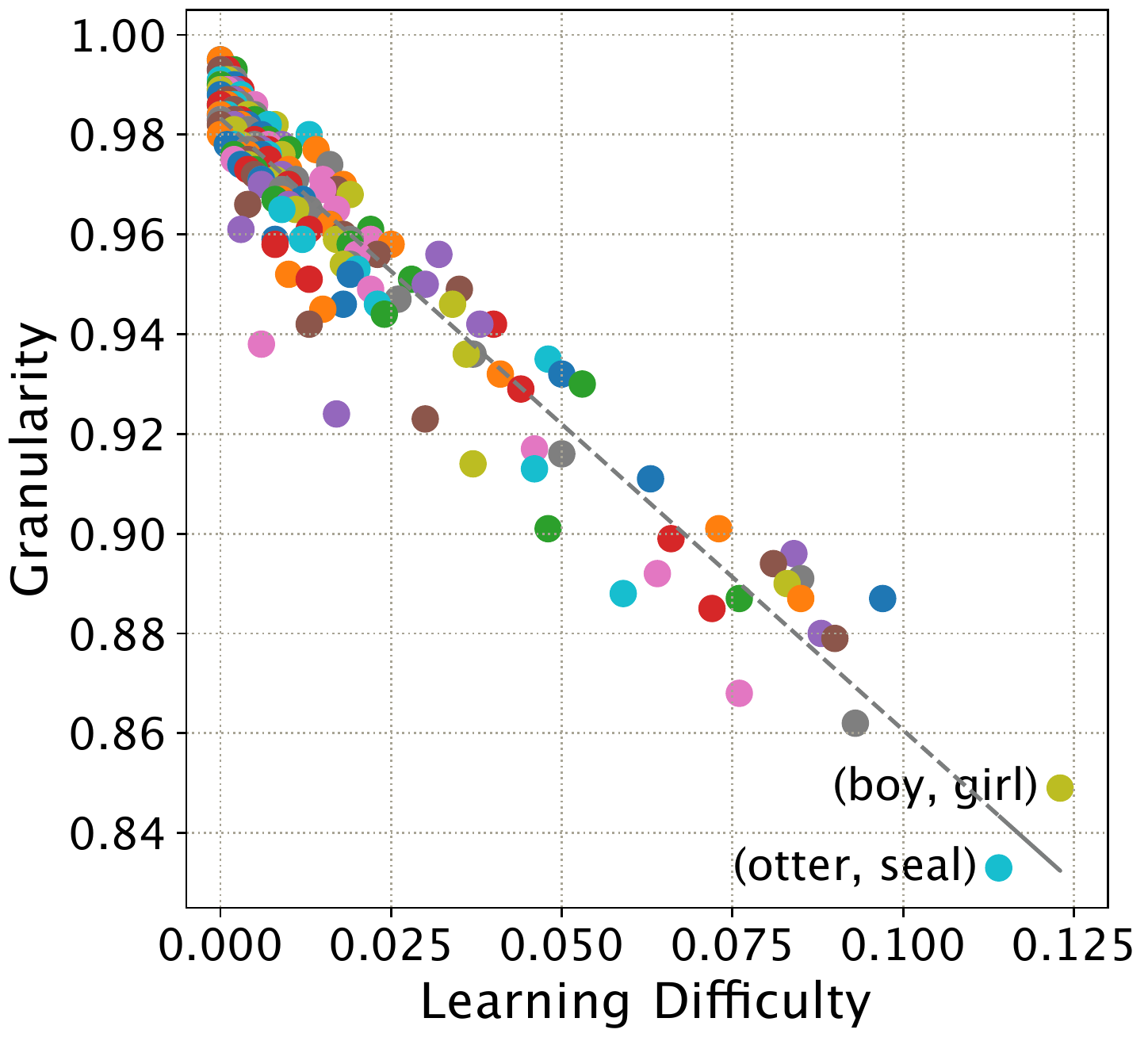}
   \label{fig:difficulty-cifar20-pair}}
\end{center}
\caption{Training difficulty versus granularity on all pairs of classes from CIFAR. The number in the bracket denotes the Pearson's $\rho$ correlation coefficient.}
\label{fig:difficulty-cifar-pair}
\end{figure}

\begin{figure}[t]
\begin{center}
\subfigure[\scriptsize{CIFAR10-Sweet (0.997)}]{
\includegraphics[width=0.35\columnwidth]{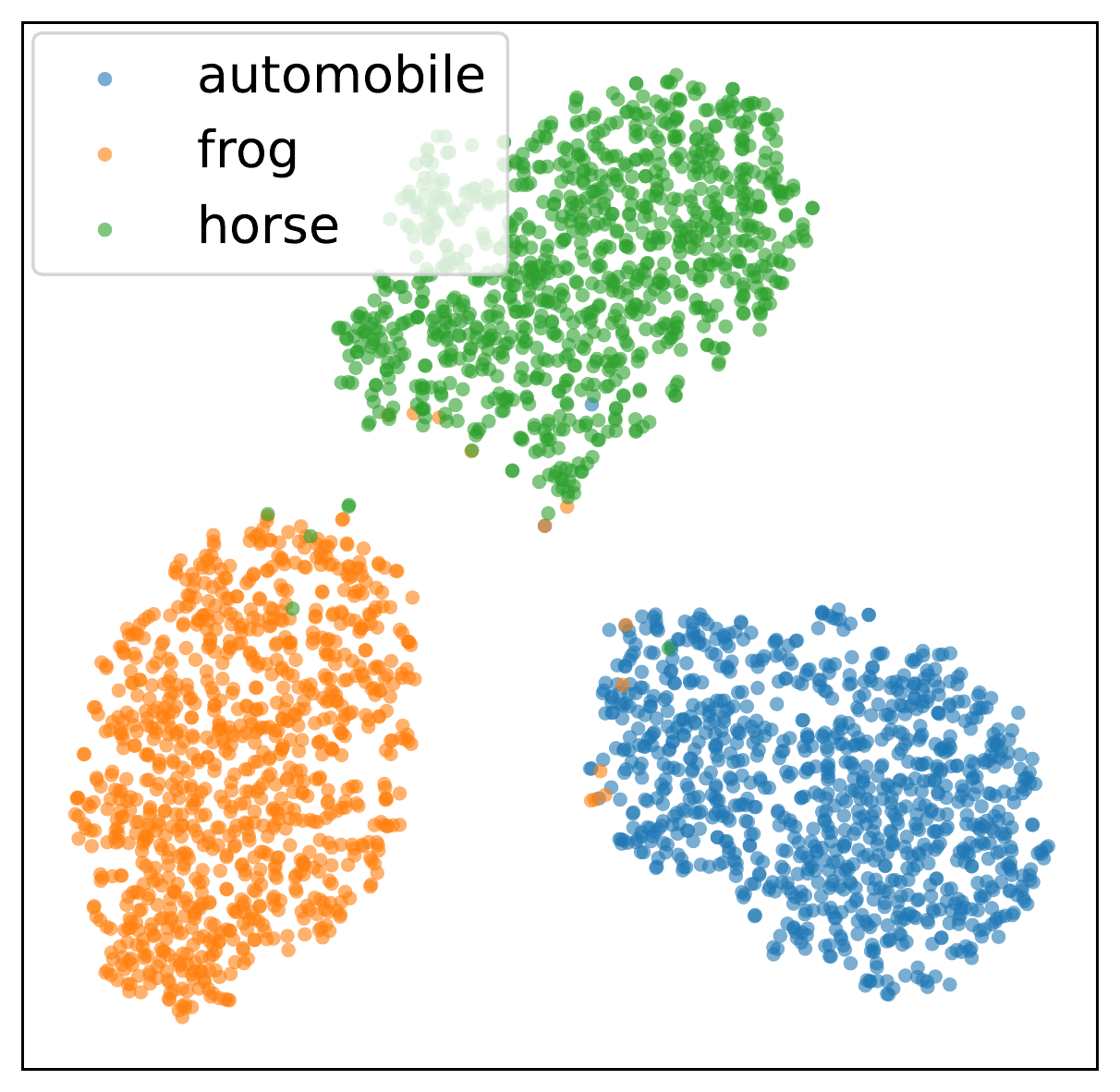}}
\subfigure[\scriptsize{CIFAR10-Bitter (0.927)}]{
\includegraphics[width=0.35\columnwidth]{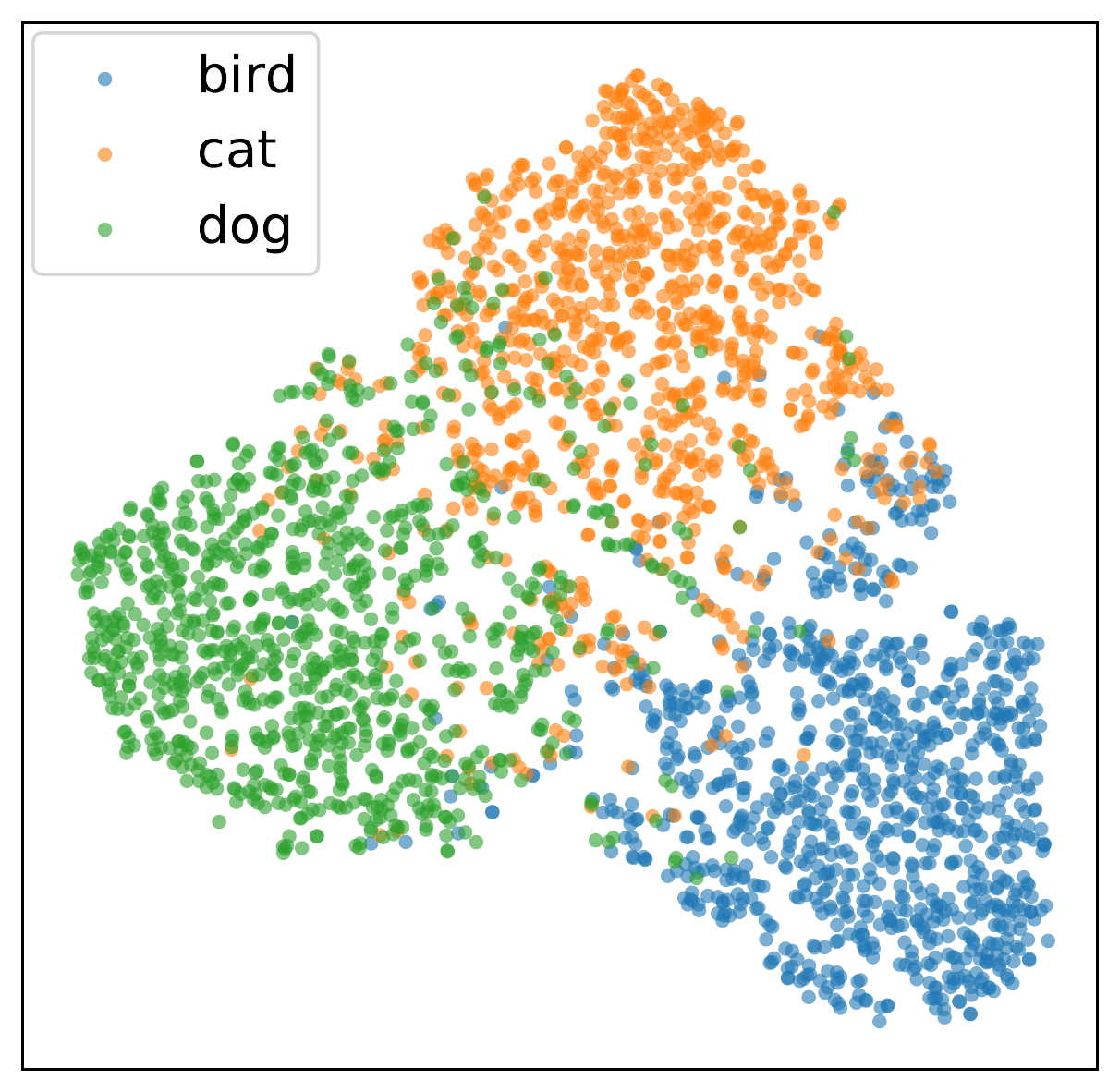}}
\subfigure[\scriptsize{CIFAR100-Sweet (0.979)}]{
\includegraphics[width=0.35\columnwidth]{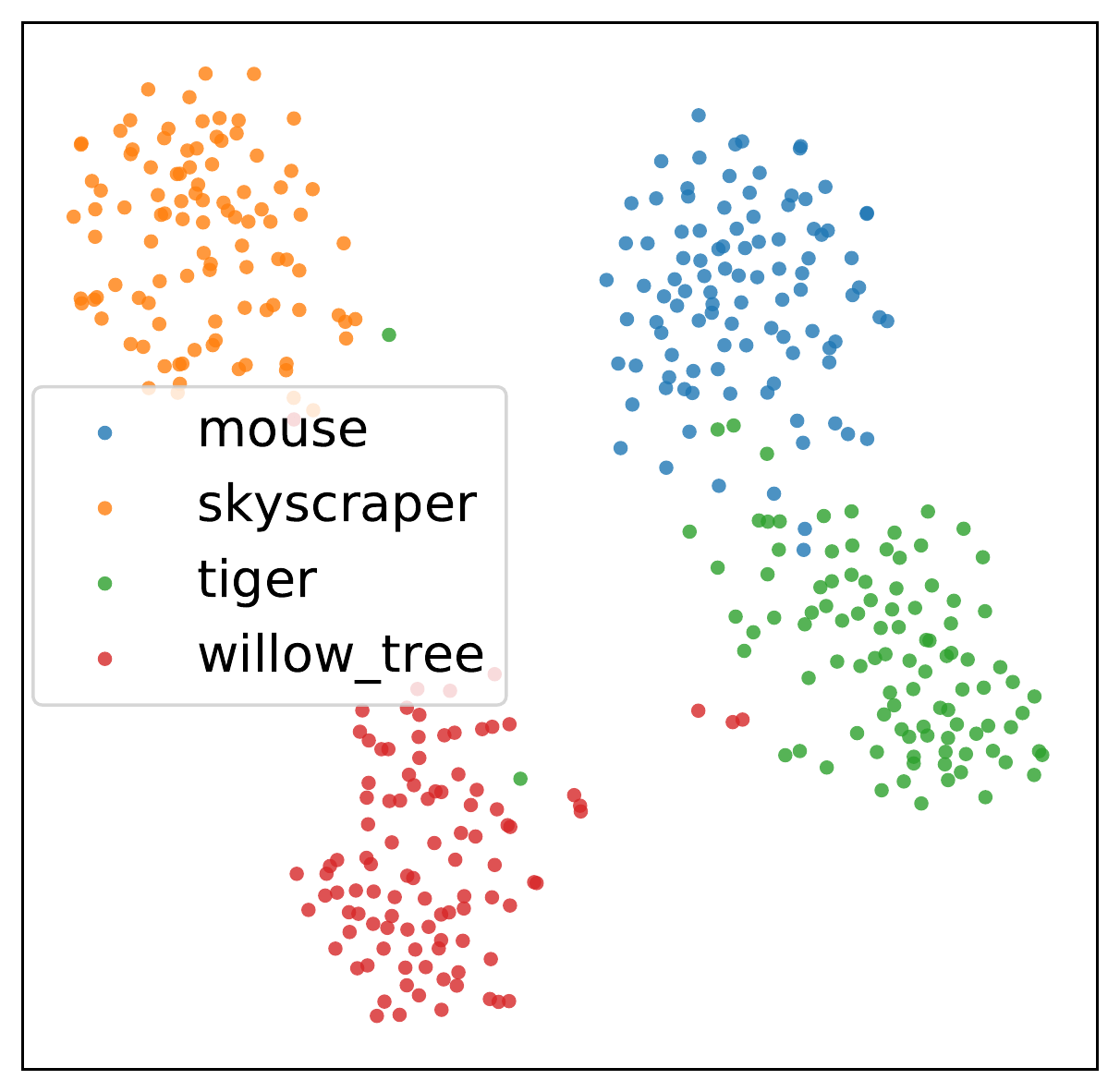}}
\subfigure[\scriptsize{CIFAR100-Bitter (0.785)}]{
\includegraphics[width=0.35\columnwidth]{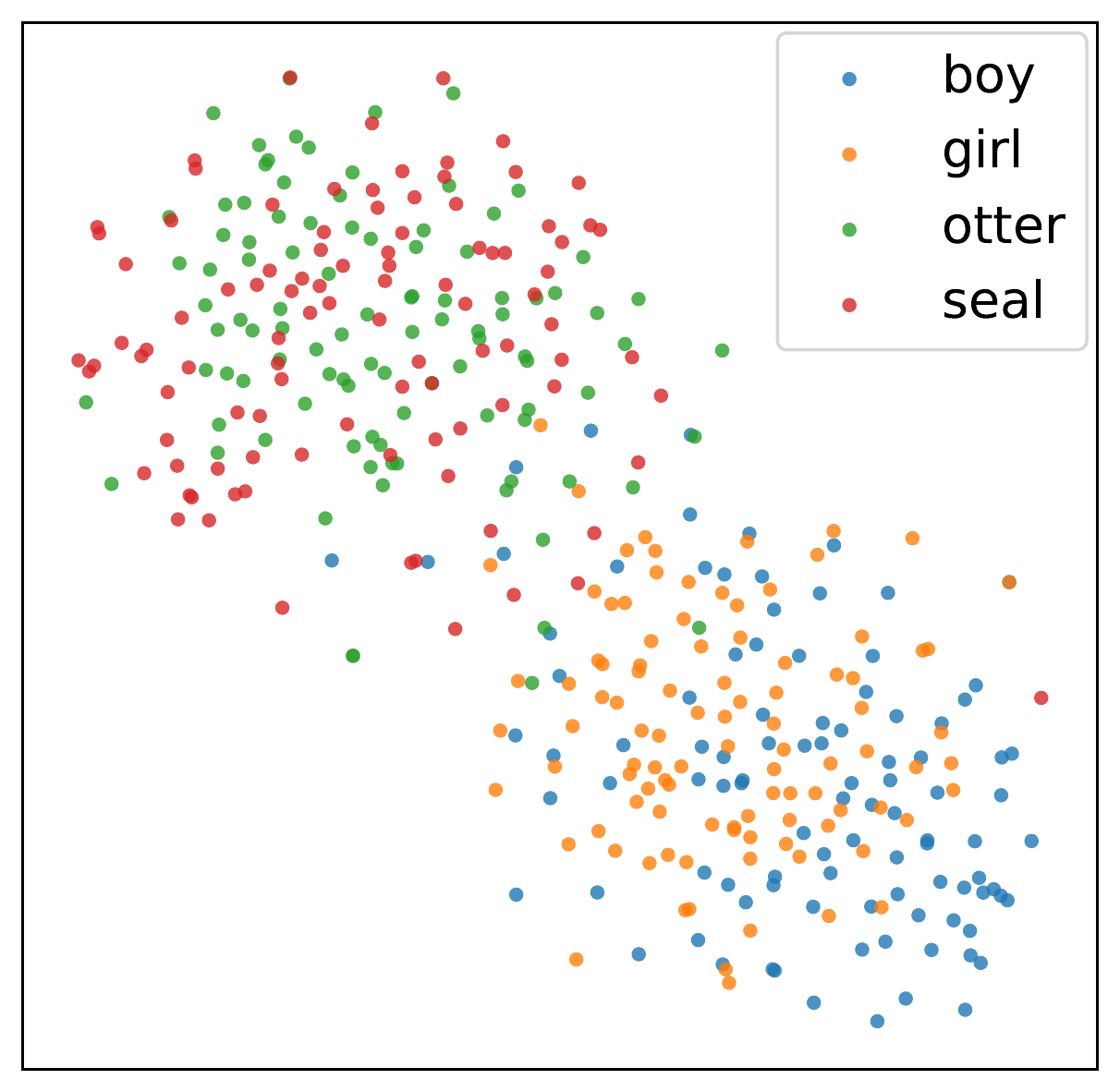}}
\end{center}
\caption{t-SNE embedding of subsets from CIFAR-10 and CIFAR-100. The granularity is shown in the bracket.}
\label{fig:tsne-cifar}
\end{figure}

\section*{Appendix C: Sensitivity to Nuisance Factors}
\label{sec:exp_nuisance}
\begin{figure*}[t]
\begin{center}
\subfigure[Gaussian Noise]{
\includegraphics[width=0.67\columnwidth]{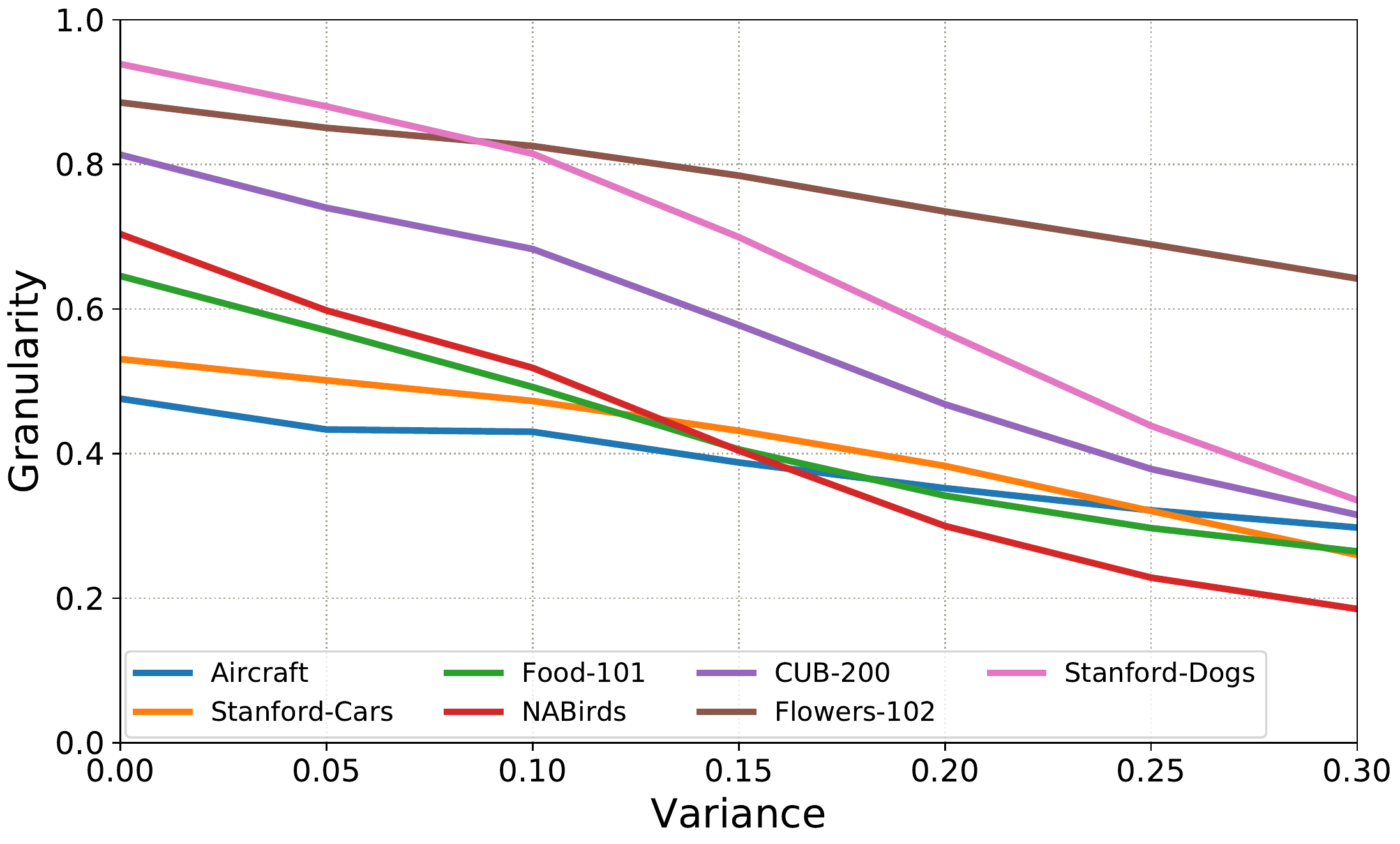}
  \label{fig:nuisance-gaussian}}
\subfigure[Salt \& Pepper Noise]{
\includegraphics[width=0.67\columnwidth]{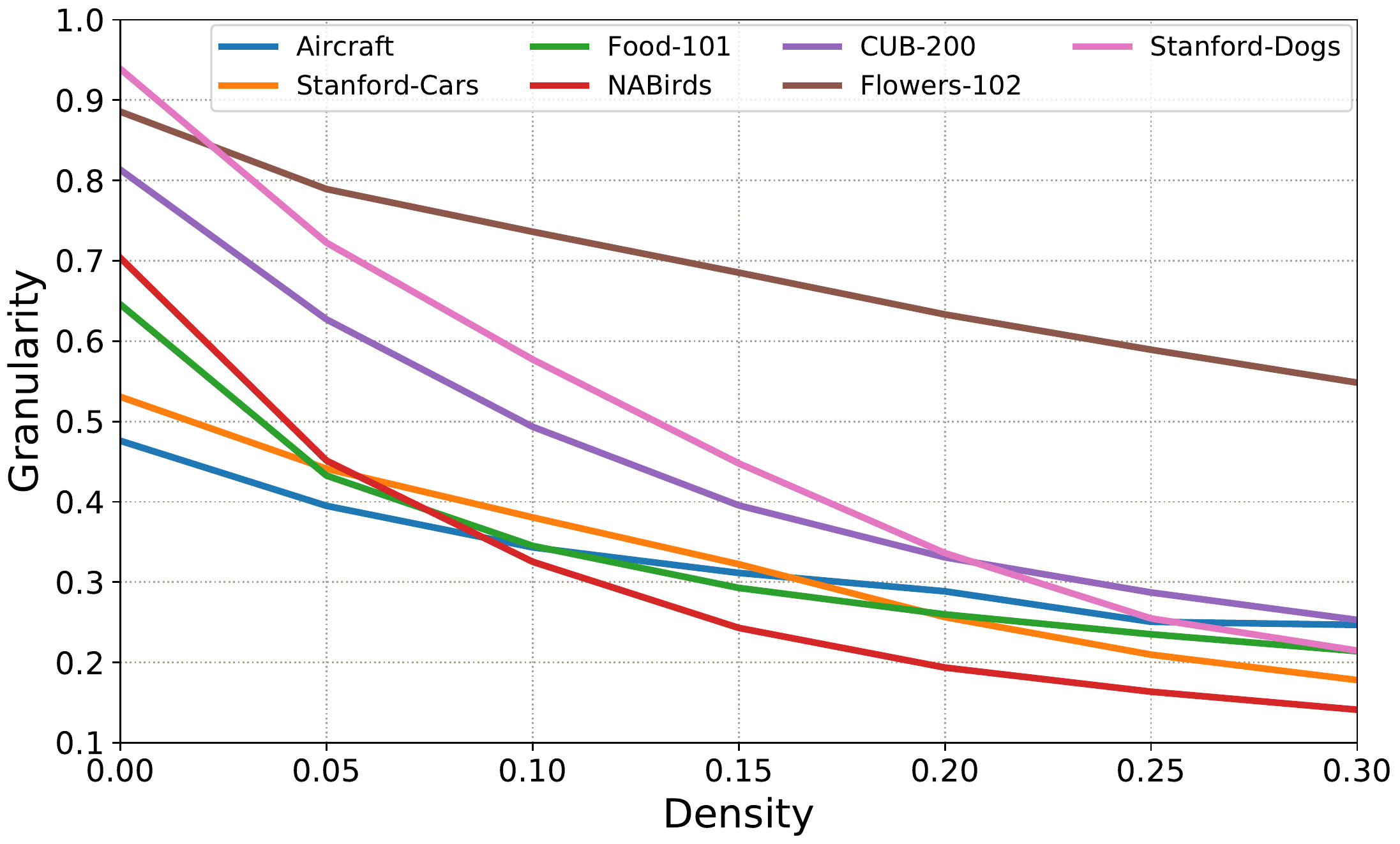}
  \label{fig:nuisance-sp}}
\subfigure[Reduced Resolution]{
\includegraphics[width=0.67\columnwidth]{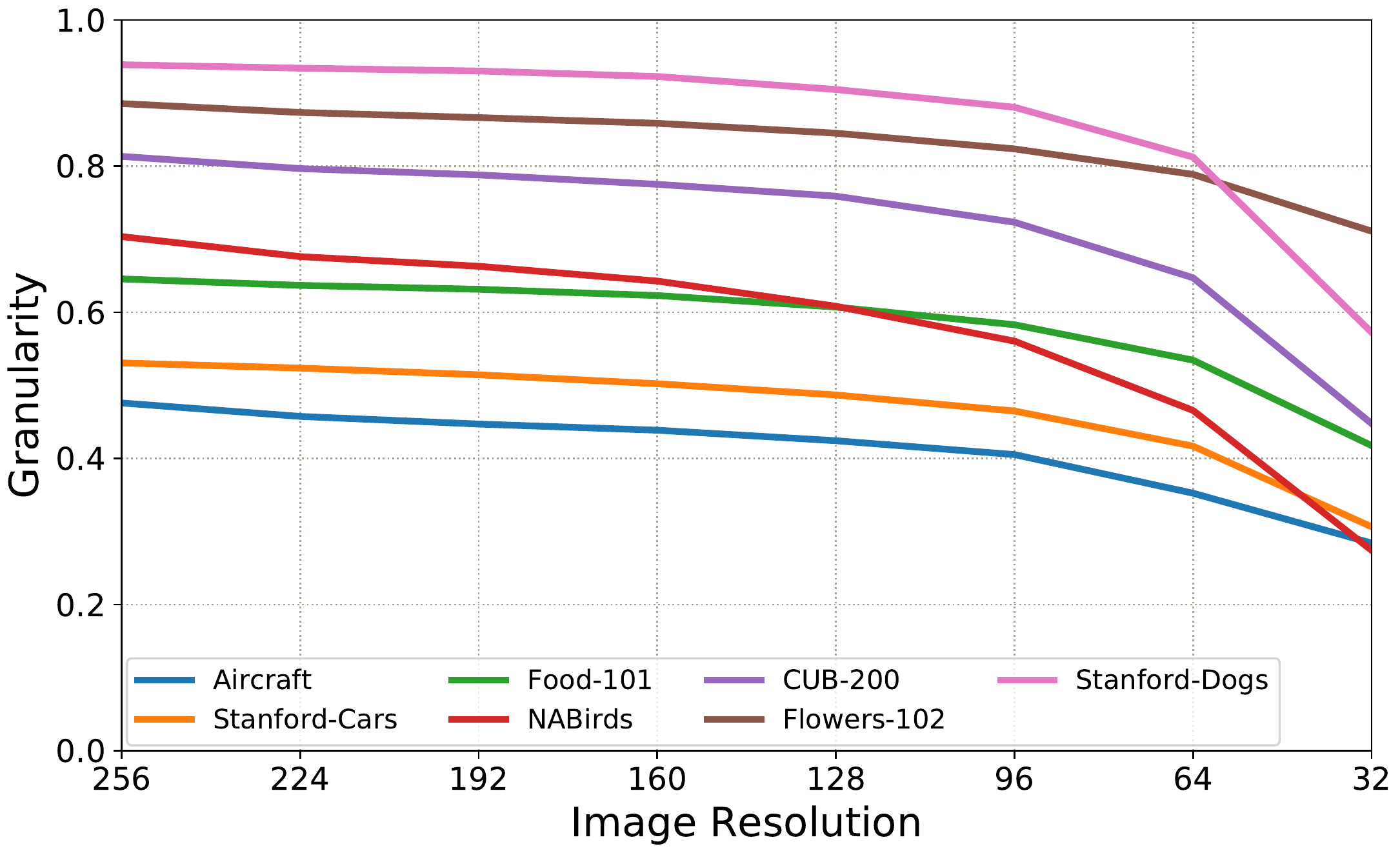}
  \label{fig:nuisance-res}}
\end{center}
\caption{Dataset granularity on different datasets with respect to nuisance factors including Gaussian noise, Salt \& Pepper noise and reduced resolution. The datasets become more fine-grained when we increase the level of noise and reduce the image resolution.}
\label{fig:nuisance}
\end{figure*}

Ideally, to distinguish classes with subtle visual difference, one could zoom in and leverage information from local discriminative regions.
However, nuisance factors arise during image acquisition in real-world (\eg, noise, motion blur, occlusion \etc).
Furthermore, due to physical constraints of imaging devices, one can not get sharp images with arbitrary high spatial resolution.
To examine how nuisance factors affects dataset granularity, we add Gaussian noise and Salt \& Pepper noise to images before feature extraction.
In addition, we reduce the image resolution by re-sizing the shorter edge of the image and keeping the aspect ratio.
Fig.~\ref{fig:nuisance} shows dataset granularity on 7 datasets using ImageNet pre-trained ResNet-50 features.
Dataset granularity is sensitive to noise and reduced resolution.
In addition, nuisance factors may change the relative order of granularity. For example, NABirds becomes more fine-grained than Aircraft, Stanford Cars and Food-101 as we increasing the noise and reducing the image resolution.

\section*{Appendix D: Proof}
\label{sec:proof}

\begin{proposition}[]
\label{prop:2}
Dataset granularity measures Fisher, RS, RSM, Rank and RankM satisfy granularity consistency, isomorphism invariance and scale invariance.
\end{proposition}

\begin{proof}
We add a prime symbol ($'$) to denote the variable after transformation.
For the proof of granularity consistency, we suppose distance function $d'$ is a granularity consistent transformation of distance function $d$; 
for isomorphism invariance, we suppose dataset $\mathcal{S}'$ is a isomorphic transformation of dataset $\mathcal{S}$;
and for scale invariance, we suppoe $d' = \alpha d$.
We denote $x_i, x_j$ in the same class by $x_i \sim_c x_j$ and $x_i, x_j$ in different classes by $x_i \nsim_c x_j$, respectively.

\textit{Fisher is granularity consistent}:
We have $d'(c_i, c_j) \geq d(c_i, c_j)$ and $d'(x_i, c_{x_i}) \leq d(x_i, c_{x_i})$ because $c_i \nsim_c c_j$ and $x_i \sim_c c_{x_i}$.
Therefore, $\frac{\sum_{i \neq j} d'(c_i, c_j)}{\sum_{i=1}^n d'(x_i, c_{x_i})} \geq \frac{\sum_{i \neq j} d(c_i, c_j)}{\sum_{i=1}^n d(x_i, c_{x_i})}$, $\text{Fisher}(\mathcal{S}, d') \geq \text{Fisher}(\mathcal{S}, d)$.

\textit{Fisher is isomorphism invariant}:
For each sample, the medoid of each class remains unchanged under the isomorphism transformation. Therefore, $d(c_i, c_j)$ and $d(x_i, c_{x_i})$ remain unchanged, $\text{Fisher}(\mathcal{S}', d) = \text{Fisher}(\mathcal{S}, d)$.

\textit{Fisher is scale invariant}:
Since $\frac{\sum_{i \neq j} d'(c_i, c_j)}{\sum_{i=1}^n d'(x_i, c_{x_i})} = \frac{\sum_{i \neq j} \alpha d(c_i, c_j)}{\sum_{i=1}^n \alpha d(x_i, c_{x_i})} = \frac{\sum_{i \neq j} d(c_i, c_j)}{\sum_{i=1}^n d(x_i, c_{x_i})}$, we have $\text{Fisher}(\mathcal{S}, d') = \text{Fisher}(\mathcal{S}, d)$.

\textit{RS is granularity consistent}:
Since $d'(x_i, x) \leq d(x_i, x)$ if $x_i \sim_c x$ and $d'(x_i, x) \geq d(x_i, x)$ if $x_i \nsim_c x$, we have $a'(x_i) \leq a(x_i)$ and $b'(x_i) \geq b(x_i)$. Therefore, $\frac{b'(x_i)}{a'(x_i)} \geq \frac{b(x_i)}{a(x_i)}$, $\text{RS}(\mathcal{S}, d') \geq \text{RS}(\mathcal{S}, d)$.

\textit{RS is isomorphism invariant}:
We have $a'(x_i) = a(x_i)$ and $b'(x_i) = b(x_i)$.
Therefore, $\text{RS}(\mathcal{S}', d) = \text{RS}(\mathcal{S}, d)$.

\textit{RS is scale invariant}: 
We have $a'(x_i) = \alpha a(x_i)$ and $b'(x_i) = \alpha b(x_i)$. Therefore, $\frac{a'(x_i)}{b'(x_i)} = \frac{a(x_i)}{b(x_i)}$, $\text{RS}(\mathcal{S}, d') = \text{RS}(\mathcal{S}, d)$.

\textit{RSM is granularity consistent}:
We have $d'(x_i, c_{x_i}) \leq d(x_i, c_{x_i})$ and $d'(x_i, c'_{x_i}) \geq d(x_i, c'_{x_i})$ because $x_i \sim_c c_{x_i}$ and $x_i \nsim_c c'_{x_i}$.
Therefore, $\frac{d'(x_i, c'_{x_i})}{d'(x_i, c_{x_i})} \geq \frac{d(x_i, c'_{x_i})}{d(x_i, c_{x_i})}$, $\text{RSM}(\mathcal{S}, d') \geq \text{RSM}(\mathcal{S}, d)$.

\textit{RSM is isomorphism invariant}:
For each sample, the medoid of each class remains unchanged under the isomorphism transformation. Therefore, $d(x_i, c_{x_i})$ and $d(x_i, c'_{x_i})$ remain unchanged and $\text{RSM}(\mathcal{S}', d) = \text{RSM}(\mathcal{S}, d)$.

\textit{RSM is scale invariant}: 
Since $\frac{d'(x_i, c_{x_i})}{d'(x_i, c'_{x_i})} = \frac{\alpha d(x_i, c_{x_i})}{\alpha d(x_i, c'_{x_i})} = \frac{d(x_i, c_{x_i})}{d(x_i, c'_{x_i})}$, $\text{RSM}(\mathcal{S}, d') = \text{RSM}(\mathcal{S}, d)$.

\textit{Rank is granularity consistent}:
After granularity consistent transformation, no samples from other classes have rank smaller than before.
Therefore, we have $R'_{ij} \leq R_{ij}$.
Since $|R'_i| = |R_i|$, $\sum_{j=1}^{|R'_i|} \frac{j}{R'_{ij}} \geq \sum_{j=1}^{|R_i|} \frac{j}{R_{ij}}$ and $\text{Rank}(\mathcal{S}, d') \geq \text{Rank}(\mathcal{S}, d)$.

\textit{Rank is isomorphism invariant}:
For each sample, the rank list $R_{ij}$ remains unchanged under the isomorphism transformation. Therefore, $R'_{ij} = R_{ij}$ and $\text{Rank}(\mathcal{S}', d) = \text{Rank}(\mathcal{S}, d)$.

\textit{Rank is scale invariant}:
For each sample, the rank list $R_{ij}$ remains unchanged when $d' = \alpha d$, we have $R'_{ij} = R_{ij}$. Therefore, $\text{Rank}(\mathcal{S}, d') = \text{Rank}(\mathcal{S}, d)$.

\textit{RankM is granularity consistent}:
After granularity consistent transformation, we have the rank $R'_{ic} \leq R_{ic}$.
Therefore, $\frac{1}{R'_{ic}} \geq \frac{1}{R_{ic}}$ and $\text{RankM}(\mathcal{S}, d') \geq \text{RankM}(\mathcal{S}, d)$.

\textit{RankM is isomorphism invariant}:
For each sample, the rank $R_{ic}$ remains unchanged under the isomorphism transformation.
Therefore, $R'_{ic} = R_{ic}$ and $\text{RankM}(\mathcal{S}', d) = \text{RankM}(\mathcal{S}, d)$.

\textit{RankM is scale invariant}:
For each sample, the rank of $R_{ic}$ remains unchanged when $d' = \alpha d$, we have $R'_{ic} = R_{ic}$. Therefore, $\text{RankM}(\mathcal{S}, d') = \text{RankM}(\mathcal{S}, d)$.
\end{proof}

\section*{Appendix E: Other Measures}

\textbf{The Baker-Hubert Gamma index (BHG).}
BHG was proposed by Baker and Hubert~\cite{baker1975measuring}, as an adaptation of the Goodman-Kruskal Gamma measure~\cite{goodman1979measures} for rank correlation between two vectors.
BHG is calculated as~\footnote{Note that the original form of BHG is $(N^+ - N^-)/(N^+ + N^-)$, we change it to $N^+/N^-$ to meet the requirement of granularity consistency.}:
\begin{equation}
    \text{BHG}(\mathcal{S}, d) = \frac{1}{n} \sum_{i=1}^n \frac{N^+}{N^-},
\end{equation}
where $N^+$ denotes the number of times a distance between two samples from the same class is smaller than the distance between two samples from different classes and $N^-$ represents the number of times in the opposite case.
Formally, 
\begin{equation}
    N^+ = \sum_{x_i \sim_c x_j, x_i' \nsim_c x_j'}\mathds{1}\big(d(x_i, x_j) < d(x_i', x_j')\big),
\end{equation}
\begin{equation}
    N^- = \sum_{x_i \sim_c x_j, x_i' \nsim_c x_j'}\mathds{1}\big(d(x_i, x_j) > d(x_i', x_j')\big),
\end{equation}
where $\mathds{1}(.)$ is an indicator function.

\begin{proposition}[]
\label{prop:3}
BHG satisfies granularity consistency, isomorphism invariance and scale invariance.
\end{proposition}

\begin{proof}
$N^+$ and $N^-$ remain unchanged when $d' = \alpha d$ and under isomorphism transformation, so BHG satisfies scale invariance and isomorphism invariance.
For granularity consistency, since $x_i \sim_c x_j$ and $x_i' \nsim_c x_j'$, we have $d'(x_i, x_j) \leq d(x_i, x_j)$ and $d'(x_i', x_j') \geq d(x_i', x_j')$.
Therefore, $N'^+ \geq N^+$ and $N'^- \leq N^-$, $\text{BHG}(\mathcal{S}, d') \geq \text{BHG}(\mathcal{S}, d)$.
\end{proof}

\textbf{The C index (C).}
The C index was proposed by Hubert and Levin~\cite{hubert1976general} as a clustering quality measure.
Let $D = \sum_{x_i \sim_c x_j} d(x_i, x_j)$ denotes the sum of distances between all pairs of samples from the same class, and there are $N = \sum_{i=1}^k n_i(n_i-1)/2$ such pairs, where $n_i$ is the number of samples in class $i$.

The C index is defined as:
\begin{equation}
    \text{C}(\mathcal{S}, d) = \frac{D_{\max} - D_{\min}}{D - D_{\min}},
\end{equation}
where $D_{\max}$ is the sum of the $N$ largest distances between all pairs of samples in the dataset and $D_{\min}$ is the sum of the $N$ smallest distances between all pairs.

\begin{proposition}[]
\label{prop:4}
The C index satisfies granularity consistency, isomorphism invariance and scale invariance.
\end{proposition}

\begin{proof}
When $d' = \alpha d$, $D = \alpha D$, $D_{\min} = \alpha D_{\min}$ and $D_{\max} = \alpha D_{\max}$, so $\text{C}(\mathcal{S}, d') = \frac{\alpha D_{\max} - \alpha D_{\min}}{\alpha D - \alpha D_{\min}} = \frac{D_{\max} - D_{\min}}{D - D_{\min}} = \text{C}(\mathcal{S}, d)$.
Therefore, the C index is scale invariant.
The C index is also isomorphism invariant since it does not depend on class indexes.

Suppose the distance between $p$ pairs changed after the granularity consistent transformation.
This transformation can be viewed as a series of $p$ granularity consistent transformations, each changed only 1 pair.
We consider such transformation for a single pair $x_i, x_j$ in this proof.
Denote the set of pairs included in the calculation of $D_{\max}$ as $S_{\max}$ and the set of pairs included in the calculation of $D_{\min}$ as $S_{\min}$, there are 4 cases in terms of which subset the pair belongs to:

\underline{Case 1}: $x_i, x_j \in S_{\min}$ and $x_i, x_j \notin S_{\max}$.

1) If $x_i \sim_c x_j$, $d'(x_i,x_j) = d(x_i,x_j) - \epsilon$, where $\epsilon > 0$ denotes the difference between the pairwise distance before and after the granularity consistent transformation, we have:
\begin{align*}
    \text{C}(\mathcal{S}, d') &= \frac{D_{\max} - (D_{\min} - \epsilon)}{(D - \epsilon) - (D_{\min} - \epsilon)} \\
    &= \frac{D_{\max} - D_{\min} + \epsilon}{D - D_{\min}} > \text{C}(\mathcal{S}, d).
\end{align*}

2) If $x_i \nsim_c x_j$, $d'(x_i,x_j) = d(x_i,x_j) + \epsilon$. We have:
\begin{align*}
    \text{C}(\mathcal{S}, d') &= \frac{D_{\max} - (D_{\min} + \epsilon)}{D - (D_{\min} + \epsilon)} \\
    &= \frac{D_{\max} - D_{\min} - \epsilon}{D - D_{\min} - \epsilon} > \text{C}(\mathcal{S}, d).
\end{align*}
This is because $\frac{D_{\max} - D_{\min}}{D - D_{\min}} > 1$.

\underline{Case 2}: $x_i, x_j \in S_{\max}$ and $x_i, x_j \notin S_{\min}$. 

1) If $x_i \sim_c x_j$:
\begin{align*}
    \text{C}(\mathcal{S}, d') &= \frac{(D_{\max} - \epsilon) - D_{\min}}{(D - \epsilon) - D_{\min}} > \text{C}(\mathcal{S}, d).
\end{align*}

2) If $x_i \nsim_c x_j$:
\begin{align*}
    \text{C}(\mathcal{S}, d') &= \frac{(D_{\max} + \epsilon) - D_{\min}}{D - D_{\min}} > \text{C}(\mathcal{S}, d).
\end{align*}

\underline{Case 3}: $x_i, x_j \in S_{\max} \cap S_{\min}$. 

1) If $x_i \sim_c x_j$:
\begin{align*}
    \text{C}(\mathcal{S}, d') &= \frac{(D_{\max} - \epsilon) - (D_{\min} - \epsilon)}{(D - \epsilon) - (D_{\min} - \epsilon)} = \text{C}(\mathcal{S}, d).
\end{align*}

2) If $x_i \nsim_c x_j$:
\begin{align*}
    \text{C}(\mathcal{S}, d') &= \frac{(D_{\max} + \epsilon) - (D_{\min} + \epsilon)}{D - (D_{\min} + \epsilon)}\\ &= \frac{D_{\max} - D_{\min}}{D - D_{\min} - \epsilon} > \text{C}(\mathcal{S}, d).
\end{align*}

\underline{Case 4}: $x_i, x_j \notin S_{\max}$ and $x_i, x_j \notin S_{\min}$.

1) If $x_i \sim_c x_j$:
\begin{align*}
    \text{C}(\mathcal{S}, d') &= \frac{D_{\max} - D_{\min}}{(D - \epsilon) - D_{\min}} > \text{C}(\mathcal{S}, d).
\end{align*}

2) If $x_i \nsim_c x_j$:
\begin{align*}
    \text{C}(\mathcal{S}, d') &= \frac{D_{\max} - D_{\min}}{D - D_{\min}} = \text{C}(\mathcal{S}, d).
\end{align*}

Considering all 4 cases, we have $\text{C}(\mathcal{S}, d') \geq \text{C}(\mathcal{S}, d)$.
The C index is granularity consistent.
\end{proof}

{\small
\bibliographystyle{ieee_fullname}
\bibliography{ref}
}

\end{document}